\documentclass[runningheads,envcountsame]{llncs}
\usepackage[T1]{fontenc}
\usepackage{graphicx}

\usepackage[bookmarks,unicode,colorlinks=true]{hyperref}%
   \def\@citecolor{blue}%
   \def\@urlcolor{blue}%
   \def\@linkcolor{blue}%
\def\orcidID#1{\smash{\href{http://orcid.org/#1}{\protect\raisebox{-1.25pt}{\protect\includegraphics{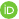}}}}}
\usepackage{color}

\usepackage[dvipsnames]{xcolor}
\usepackage{amsmath}    
\usepackage{amssymb}
\usepackage{graphicx}
\usepackage{subcaption}
\usepackage[ruled,linesnumbered]{algorithm2e}
\usepackage{enumerate}
\usepackage{cite}
\usepackage{tikz}
\usepackage{pifont}%
\usepackage{multirow}

\usepackage{thmtools}
\usepackage{thm-restate} 

\usetikzlibrary{arrows}
\usetikzlibrary{shapes}
\usetikzlibrary{automata}
\usetikzlibrary{positioning,shadows,trees}
\tikzset{initial text={}}
\usepackage{graphicx}
\usepackage{booktabs} 
\usepackage{siunitx} 
\usepackage{tabularx}

 \usepackage{cleveref} 
\Crefname{figure}{Fig.}{Figs.}
\usepackage{lineno}

\usepackage{ltablex} 
\keepXColumns 

\newcommand{\ra}{RA\xspace}
\newcommand{\dra}{DRA\xspace}
\newcommand{\wdra}{DRA\xspace}
\newcommand{\sca}{$S$-consistent\xspace}
\newcommand{\scap}{$S$-consistent with\xspace}
\newcommand{\spcap}{$S'$-consistent with\xspace}
\newcommand{\scn}{$S$-consistency\xspace}

\newcommand{\A}{\mathcal{A}}

\newcommand{\MQ}{\textsf{MQ}}

\newcommand{\TQ}{\textsf{TQ}}

\newcommand{\EQ}{\textsf{EQ}}

\newcommand{\Mem}{\textsf{Mem}}

\newcommand{\hypothesis}{\mathcal{H}}

\newcommand{\alphabet}{\Sigma}
\newcommand{\finwords}{\alphabet^*}

\newcommand{\lang}{L}

\newcommand{\emptyword}{\varepsilon}

\newcommand{\expr}{S}
\newcommand{\obtable}{\mathcal{T}}

\newcommand{\bigO}{\mathcal{O}}

\newcommand{\ralib}{\textsf{RaLib}}

\newcommand{\alphabetQ}{\mathbb{Q}}
\newcommand{\alphabetZ}{\mathbb{Z}}

\newcommand{\target}{\A}

\newcommand{\mdot}{\!\cdot\!}
\iftrue
\usepackage{todonotes}
	
	\newcommand{\dd}[1]{\todo[inline,color=teal!10,caption={DiDe}]{\textbf{Di-De:} #1}}
	\newcommand{\qiyi}[1]{\todo[inline,color=Orchid!10,caption={Qiyi}]{\textbf{Qiyi:} #1}}
    \newcommand{\ly}[1]{\todo[inline,color=orange!10,caption={LY}]{\textbf{LY:} #1}}
    
\else
\newcommand{\dd}[1]{}
\newcommand{\qiyi}[1]{}
\newcommand{\ly}[1]{}
\fi

\iffalse
\newcommand{\warn}[1]{\todo[inline,color=red!75,caption={W}]
\newcommand{\modify}[1]{{\color{red}{#1}}} {\textbf{IMPORTANT: #1}}}
\else
\newcommand{\warn}[1]{}
\newcommand{\modify}[1]{{\color{black}{#1}}} 
\fi

\begin{document}\sloppy
%
\title{Learning Canonical Register Automata over Ordered Data Domains}
%
%
\author{Yong Li\inst{1,2}\orcidID{0000-0002-7301-9234} \and
Qiyi Tang\inst{3}\orcidID{0000-0002-9265-3011} \and
Di-De Yen\inst{3}\orcidID{0000-0003-0045-9594}}
\authorrunning{Yong Li, Qiyi Tang and Di-De Yen}
%
\institute{Key Laboratory of System Software (Chinese Academy of Sciences), Beijing, China
\and 
Institute of Software, Chinese Academy of Sciences, Beijing, China
\\
\email{liyong@ios.ac.cn}\and
School of Computer Science and Informatics, University of Liverpool, UK\\
\email{\{qiyi.tang, d.d.yen\}@liverpool.ac.uk}}
\maketitle              
%

\begin{abstract}
Register automata are finite automata equipped with memory that recognize data languages over infinite alphabets. In this work, we investigate active learning algorithms for deterministic register automata (DRAs) over \emph{ordered} data domains---covering both dense domains, such as the rationals, and non-dense domains such as the integers. We show that the active learning problem for DRAs over both dense and non-dense ordered domains can be treated within a single unified framework. 
More specifically, we develop and implement a polynomial-time active learning procedure for DRAs over ordered domains, using oracles for membership, equivalence and memorability queries. 
The memorability queries were originally introduced for learning DRAs over domains with identity tests. 
Our unified framework also leads to a new consequence: minimization of DRAs over the non-dense ordered domain of integers is decidable, extending a result previously known only for dense domains.
Finally, we give improved complexity bounds of several decision problems for DRAs over ordered domains that are closely related to the queries used in active learning.

\end{abstract}

\section{Introduction}\label{sec:intro}

The $L^*$ algorithm, introduced by Angluin~\cite{DBLP:journals/iandc/Angluin87} within the Minimally Adequate Teacher (MAT) framework, enables a learner to infer an automaton representing an unknown system by interacting with a teacher through two types of queries: membership queries and equivalence queries. 
This query-based paradigm is often referred to as \emph{active learning}~\cite{DBLP:journals/iandc/Angluin87,DBLP:conf/tacas/VaandragerGRW22}.
A membership query asks whether a word $u$ belongs to the target language~$L$, while an equivalence query asks whether a proposed automaton recognizes~$L$. 
After collecting enough information from membership queries, the learner constructs a conjectured automaton and asks an equivalence query. 
The teacher either confirms that the conjecture is correct or returns a counterexample, which is then used by the learner to refine the conjecture. This iterative process continues until the learner successfully converges to the target automaton.

Active learning has been widely applied, including learning assumptions for compositional verification~\cite{DBLP:conf/tacas/CobleighGP03}, software testing~\cite{Czerny2014LearningbasedST}, detecting errors in network protocol implementations~\cite{DBLP:conf/uss/RuiterP15}, regular model checking~\cite{DBLP:conf/fmcad/ChenHLR17}, extracting automata models for recurrent neural networks~\cite{DBLP:conf/icml/WeissGY18}, verifying binarized neural networks~\cite{DBLP:conf/sat/ShihDC19}, and explaining machine learning models~\cite{DBLP:conf/ecai/Ozaki24}.

In this paper, we consider active automata learning for \emph{deterministic register automata} (DRAs) over \emph{ordered} data domains. 
Register automata (RAs) are automata over infinite alphabets and have been extensively studied~\cite{Kaminski:94:TCS, Neven:04:ACMCL, Ley:10:TR, DBLP:conf/mfcs/BalachanderFGT25, Alur:13:ACM, ChenOTW2017 } motivated by real-world problems where data may have an unbounded domain. 
An RA stores data symbols in a finite number of registers, which can later be compared against new inputs. 
RAs are far more expressive than classical finite automata; however, they lose many desirable closure and decidability properties—language inclusion and equivalence become undecidable for nondeterministic RAs~\cite{Kaminski:94:TCS}. 
Fortunately, these problems are decidable for DRAs~\cite{Neven:04:ACMCL}.
A key challenge in learning DRAs is that the learner must simultaneously reason about control-state behavior and data constraints, unlike in classical active learning.

In particular, comparisons between input data and register values come in two forms:
\modify{(1)} identity tests, which check equality; and \modify{(2)} ordered comparisons, which in addition to equality also test whether the input is greater or smaller than a stored value.  
Most existing work focuses on DRAs restricted to identity tests~\cite{Kaminski:94:TCS, Neven:04:ACMCL, DBLP:conf/vmcai/HowarSJC12, Bollig:13:DLT}, which limits their ability to capture ordering-based properties. 
For instance, checking whether an input sequence is sorted (i.e.\ in non-decreasing order) is recognizable by a DRA with ordered comparison but \emph{not} by any RA using identity tests only.
\modify{
There has been prior work~\cite{DBLP:journals/fac/CasselHJS16, Cassel2015RALibA} on learning DRAs over ordered domains; however, it relies on tree-based queries over abstract symbolic words, whereas our approach operates directly on concrete words using standard membership queries together with memorability information.
}

We focus on a family of DRAs introduced in~\cite{Ley:10:TR} that support ordered comparisons.  
These automata prohibit duplicate register values and, after each transition, permute registers so that their ordering reflects the order of last occurrence of their stored data.  
The languages recognized by such DRAs—often called \emph{regular data languages}—admit a Myhill--Nerode-style characterization and have a canonical minimal representation that is simultaneously minimal in both the number of states and the number of registers~\cite{Ley:10:TR,Ley:11:thesis}.
This property is crucial for active learning, which relies on a canonical representation of the target formalism, analogous to the Myhill–Nerode theorem~\cite{Myhill57,Nerode58} that guarantees a unique minimal DFA for regular languages.


A subtle issue arises when comparing dense data domains (e.g.\ $\alphabetQ$) and non-dense domains (e.g.\ $\alphabetZ$).  
For RAs restricted to identity tests, this distinction is irrelevant. 
However, for DRAs with ordered comparisons, non-dense domains introduce additional complications.   

\paragraph*{\textbf{Contributions.}} We show that minimal DRAs over~$\alphabetQ$ and minimal DRAs over~$\alphabetZ$ that recognize the same language when restricted to integers are structurally equivalent.   
Consequently, the decidable complexity result of the minimization algorithm of~\cite{Ley:11:thesis}, originally developed for dense domains, extends to non-dense domains as well. 
Accordingly, we obtain a unified active learning framework that applies uniformly to both dense and non-dense data domains.
As a related result, we establish that the configuration reachability problem over $\alphabetZ$ is decidable.
Based on the canonical form of DRAs, we develop an active learning algorithm for DRAs over \emph{ordered} data domains.
This algorithm uses \emph{polynomially} many membership, equivalence and \emph{memorability} queries, which precisely identify which symbols must be stored in registers when processing a data word\footnote{With only classical membership and equivalence queries, a learning algorithm would likely require exponential time, as discussed in Sect.~\ref{sec:conclusion}}.
The main technical challenge is that, unlike learning DRAs with identity tests~\cite{DBLP:conf/vmcai/HowarSJC12, Bollig:13:DLT}, which rely on equality-based reasoning to capture freshness of values, ordered comparisons require reasoning modulo order-preserving renamings of data values.
Such renamings are not unique, but we show that a canonical representative suffices to decide the relevant equivalence relations between words (cf. Lemma~\ref{lem:inequivalence-S}), making the learning independent of the choice of mapping.
This breaks standard observation-table techniques based on suffix reasoning, which compare rows via equality of abstract future behavior.

To address this, table entries here are compared modulo canonical renamings and the subset of values stored in memory, yielding our notion $S$-consistency (cf. \Cref{sec:learning}), which is not an equivalence relation and requires refined counterexample handling that updates symbolic transitions rather than only separating states.
We implement this learning algorithm
and demonstrate its effectiveness on a family of data languages and randomly generated DRAs.

We further study the complexity of decision problems for DRAs over ordered data domains, which are closely related to the queries arising in learning frameworks.
We show PSPACE-completeness of language inclusion/intersection non-emptiness for DRAs over equality and ordered domains. For dense ordered domains, we additionally show PSPACE upper bounds for configuration equivalence, language equivalence, and memorability.
For the memorability problem, our PSPACE upper bounds improve the previously known NEXPTIME upper bounds from~\cite{Ley:10:TR, Ley:11:thesis}.

\paragraph*{\textbf{Related Work.}}
Active learning of RAs has been investigated in~\cite{DBLP:conf/vmcai/HowarSJC12,DBLP:conf/tacas/DierlFHJST24,DBLP:conf/mfcs/BalachanderFGT25, Bollig:13:DLT}. 
In~\cite{DBLP:conf/mfcs/BalachanderFGT25}, Permutation Deterministic Register Automata (PDRAs) were introduced, and our DRAs restricted to identity tests correspond to LAR-DRAs, a subclass of RAs with a fixed permutation policy. 
It was shown in~\cite{DBLP:conf/mfcs/BalachanderFGT25} that minimization of LAR-DRAs is in PTIME, as both language equivalence and memorability are tractable. 
Moreover, they provide a PTIME active learning algorithm for LAR-DRAs using memorability queries.
The learning algorithms in~\cite{DBLP:conf/vmcai/HowarSJC12,DBLP:conf/tacas/DierlFHJST24} consider the DRA model of~\cite{CASSEL201554} and do not use memorability queries; their runtime is exponential in both the number of registers and the length of the longest counterexample.
\modify{Additionally, there is work~\cite{DBLP:conf/popl/MoermanS0KS17} on learning nominal automata, which are equivalent to RAs in expressiveness~\cite{Klin:14:lmcs}. These automata group infinite states into orbits under permutations; specifically, \cite{DBLP:conf/popl/MoermanS0KS17} provides an active learning algorithm for a subclass of nondeterministic nominal automata---including deterministic ones---over infinite domains with equality tests. 
{\color{black}The authors of~\cite{DBLP:conf/popl/MoermanS0KS17} briefly note that their framework is intended to accommodate both equality and ordered domains, although working out the extension to ordered domains is left open.}
Venhoek, Moerman and Rot~\cite{DBLP:journals/tcs/VenhoekMR22} further study computation over ordered nominal sets, contributing techniques for reasoning about infinite ordered structures in nominal settings.}
A passive learning algorithm for LAR-DRAs is given in~\cite{DBLP:conf/concur/BalachanderFG24}.

Active learning for DRAs over ordered domains is studied in~\cite{DBLP:journals/fac/CasselHJS16, Cassel2015RALibA}, where data words consist of symbols from both a finite and an infinite set. Such data words are also considered in~\cite{DBLP:conf/vmcai/HowarSJC12} and can be categorized into different patterns based on their finite components. This approach relies on tree queries instead of standard membership queries: for a given concrete prefix and a set $S$ of abstract suffixes, the teacher returns a symbolic decision tree. This tree captures the behavior of the unknown automaton's runs on the prefix followed by all possible concrete suffixes that satisfy the patterns in $S$. 
This tree-query technique is also used in~\cite{DBLP:conf/tacas/DierlFHJST24} to reduce query complexity.
{\color{black}Compared with memorability queries, tree queries provide richer information, since the resulting symbolic decision trees can expose both relevant register values and transition guards. In contrast, a memorability query directly returns the values of a prefix that may affect future behavior, independently of a particular suffix, and thus provides a more specialized oracle for identifying the required register contents.}
In~\cite{DBLP:journals/fac/CasselHJS16, Cassel2015RALibA}, ordered domains with a ternary operation $\text{sum}(\ell,m,n)$ (signifying $\ell+m=n$) and those with an operation $\text{inc}(m,n)$ (signifying $m+1=n$) are also studied. Furthermore, the equivalence and emptiness problems are addressed for the more expressive model in~\cite{ChenOTW2017}, which supports ordered data domains and allows certain linear arithmetic constraints.

On the other hand, our DRAs over ordered data domains were first introduced in~\cite{Ley:10:TR,Ley:11:thesis}. In this model, the duplication of register values is prohibited; this restriction does not reduce expressiveness and ensures the uniqueness of minimal automata. In our learning algorithm, in addition to equivalence and membership queries, the teacher also answers memorability queries. These queries were introduced in~\cite{DBLP:conf/mfcs/BalachanderFGT25} to improve the complexity of learning DRAs over domains restricted to equality tests.

Symbolic finite automata (SFAs) constitute another automata model for infinite alphabets. 
Unlike RAs, SFAs retain the classical closure and decidability properties of finite automata but have strictly less expressive power~\cite{DAntoniV2017}. There is extensive work on active learning for general and restricted classes of SFAs, see, for example, ~\cite{DrewsD2017,ArgyrosD2018,FismanFZ2023,GRINCHTEIN20104029,MM15,ASKK16,Maler2017}.

\section{Preliminaries}\label{sec:pre}

We use $\mathbb{R}$ (resp.\ $\mathbb{Q}$, $\mathbb{Z}$, $\mathbb{N}$) to denote the sets of real (resp.\ rational, integer, and non-negative) numbers.
A set $S \subseteq \mathbb{R}$ is \emph{dense} in $\mathbb{R}$ if, for every $s < t$, there exists $r \in S$ with $s < r < t$~\cite{Royden:10:book}. Accordingly, $\mathbb{Q}$ is dense in $\mathbb{R}$, whereas $\mathbb{Z}$ is not. For simplicity, we say $\mathbb{Q}$ and $\mathbb{R}$ are dense and $\mathbb{Z}$ is non-dense throughout the paper.


Let $\Sigma$ be an alphabet and $R$ a binary relation on $\Sigma$, where $\Sigma$ may be finite or infinite. A \emph{(data) word} is a finite sequence over $\alphabet$.
For words $u = a_1 \dots a_m$, $v = b_1 \dots b_n$ in $\Sigma^*$ and a symbol (or letter) $d \in \Sigma$, we write $d \in u$ if $d = a_i$ for some $1 \le i \le m$, and $u \cdot v$ denotes the concatenation $a_1 \dots a_m b_1 \dots b_n$. The length of $u$ is denoted by $|u|$.
The relation $R$ induces an equivalence relation $\sim_R$ on $\Sigma^*$: for $u = a_1 \dots a_m$ and $v = b_1 \dots b_n$, we write $u \sim_R v$ if and only if (1) $m = n$, and (2) $(a_i, a_j) \in R$ if and only if $(b_i, b_j) \in R$ for all $1 \le i, j \le n$.


Each equivalence class $\tau$ of $\sim_R$ is called a \emph{$\sim_R$-word type}, or simply a \emph{type}. We say $\tau$ has length $n$ if every word $u$ of type $\tau$ has $|u| = n$. A language $L \subseteq \Sigma^*$ is a \emph{data language} over $(\Sigma, R)$ if, for all words $u, v$ of the same $\sim_R$-type, $u \in L \Leftrightarrow v \in L$. In this paper, we consider only alphabet $\alphabet \in \{\mathbb{R}, \mathbb{Q}, \mathbb{Z}\}$ and $R$ being either $<$ (the total order) or $=$ (the identity relation) on $\alphabet$.


Let $\sigma$ be a mapping on $\finwords$ with $\sigma(u \cdot v) = \sigma(u) \cdot \sigma(v)$ for all $u, v \in \finwords$. We call $\sigma$ \emph{$R$-preserving} if $\sigma(u) \sim_R u$ for all $u \in \finwords$; when $R $ is $ <$, it is \emph{order-preserving}. For any $u = a_1\dots a_n, v=b_1\cdots b_n \in \Sigma^*$ with $u \sim_R v$, there exists an $R$-preserving partial mapping $\sigma$ such that $\sigma(u) = v$.
Moreover, if $\Sigma$ is dense, $\sigma$ can be chosen bijective as follows. For $(\Sigma, R) = (\mathbb{Q}, <)$, let $u' = a'_1 \dots a'_n$ and $v' = b'_1 \dots b'_n$ be the sorted sequences of $u$ and $v$, respectively. Then $\sigma_{u \rightarrow v}$ is defined as follows:
(1) $\sigma_{u \rightarrow v}(c) = \frac{(c - a'_i)(b'_{i+1} - b'_i)}{a'_{i+1} - a'_i} + b'_i$, for each $a'_i\neq a'_{i+1}$ and $a'_i \le c < a'_{i+1}$;
(2) $\sigma_{u \rightarrow v}(a'_i) = b'_i$, for each $a'_i= a'_{i+1}$, where $i = 1, \dots, n-1$;
(3) $\sigma_{u \rightarrow v}(c) = (c - a'_1) + b'_1$, if $c < a'_1$;
(4) $\sigma_{u \rightarrow v}(c) = (c - a'_n) + b'_n$, if $c \ge a'_n$.
By definition, $\sigma_{u \rightarrow v}$ is order-preserving and bijective. Since $\sigma_{u \rightarrow v}$ is bijective, its inverse $\sigma^{-1}_{u \rightarrow v}$ also exists.


\emph{Register automata} (\ra{s}), also known as \emph{finite-memory automata}, extend finite-state automata to infinite alphabets and recognize data languages. Several equivalent variants exist in the literature~\cite{Kaminski:94:TCS, Neven:04:ACMCL, Ley:10:TR, DBLP:conf/mfcs/BalachanderFGT25}. In this paper, we adopt the definition of \ra{s} from~\cite{Ley:10:TR}.

\begin{definition}\label{def:RA}
Given $k \in \mathbb{N}$, an alphabet $\Sigma$, and a binary relation $R$ on $\Sigma$, 
a $k$-register automaton ($k$-\ra) $\A$ over $(\Sigma,R)$ 
is a tuple $(Q, q_0, F, \Delta)$~where:
(1) $Q$ is a set of states, partitioned into disjoint subsets $Q_0, \dots, Q_k$,
(2) $q_0 \in Q_0$ is the \emph{initial state},
(3) $F \subseteq Q$ is the set of \emph{final states}, and
(4) $\Delta$ is a finite set of \emph{transitions} of the form $(p, \tau, E, q)$, 
where $p \in Q_i$ and $q \in Q_j$ for some $0 \le i, j \le k$, 
$\tau$ is a $\sim_R$-word type of length $i+1$, $E \subseteq \{1, \ldots, i+1\}$, and $j=i+1-|E|$.
\end{definition}

An RA $\A$ is \emph{deterministic} (\dra) if, for any two transitions $(p, \tau, E, q)$ and $(p', \tau', E', q')$, we have $(E, q) = (E', q')$ whenever $(p, \tau) = (p', \tau')$.

A \emph{configuration} of $\A$ is a pair $(q, v)$, where $q \in Q_i$ and $v$ is a word over $\Sigma$ of length $i$, for some $0 \le i \le k$. For configurations $(p, u)$ and $(q, v)$ and a symbol $a \in \Sigma$, we write $(p, u) \xrightarrow{a}_\A (q, v)$ or $(p, u) \xrightarrow{\tau: E}_\A (q, v)$ if there exists a transition $\delta = (p, \tau, E, q)$ such that $u \cdot a = a_1 \dots a_n$ is of type $\tau$ and $v$ is obtained from $u \cdot a$ by removing all $a_i$ with $i \in E$,
\modify{
subject to the following two conditions:
(1) if $|u| = k$, then $E \neq \emptyset$ to prevent memory overflow; and
(2) $j \in E$ whenever $1 \le j < n$ and $a_j = a_n = a$, which ensures no data duplication by requiring that if the current input $a$ is already stored, it must be removed.
}
The subscript $\A$ is omitted when clear from context.

Let $w = b_1 \dots b_m \in \Sigma^*$. A sequence of configurations $\pi = (p_0, u_0) \dots (p_m, u_m)$ is a \emph{run} on $w$, written $(p_0, u_0) \xrightarrow{w} (p_m, u_m)$, if $(p_i, u_i) \xrightarrow{b_i} (p_{i+1}, u_{i+1})$ for all $0 \le i < m$. It is \emph{accepting} if $(p_0, u_0) = (q_0, \emptyword)$ and $p_m \in F$.
We say $\A$ is \emph{complete} if every word $ u\in\finwords$ has a run in $\A$ from $(q_0, \emptyword)$.
The language recognized by $\A$ is $\lang(\A) = \{ w \in \Sigma^* \mid (q_0, \emptyword) \xrightarrow{w} (q_f, u), q_f \in F, u \in \Sigma^* \}$, which is a data language over $(\Sigma, R)$. Two \ra{s} $\A$ and $\A'$ are \emph{equivalent} if $\lang(\A) = \lang(\A')$.

\begin{figure}[tp]
\centering




\includegraphics[scale=0.8]{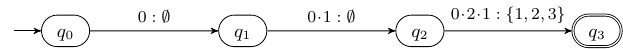}
\caption{A well-typed \dra $\mathcal{A}_{\textit{mid}}$ recognizing $L_{\textit{mid}}$ where sink state and its related transitions are omitted.
}
\label{fig:DRA_mid}
\end{figure}


Consider $L_{\textit{mid}} = \{a_1 a_2 a_3 \mid a_1 < a_3 < a_2\}$ over $(\mathbb{Z},<)$: the \dra in Fig.~\ref{fig:DRA_mid} has states $q_0, \dots, q_3$, with $q_0$ being the initial state and $q_3$ a final state, and recognizes $L_{\textit{mid}}$. Each transition is labeled with $\tau : E$. For example, the transition from $q_2$ to $q_3$ is labeled with $0 \mdot 2 \mdot 1 : \{1,2,3\}$, where $0 \mdot 2 \mdot 1$ is the word type $\tau$ and $\{1,2,3\}$ is $E$. On the word $w = 3 \cdot 9 \cdot 7$, the DRA has the accepting run $\pi = (q_0, \emptyword) \xrightarrow{3} (q_1,3) \xrightarrow{9} (q_2,3 \cdot 9) \xrightarrow{7} (q_3, \emptyword)$, so $w \in L_{\textit{mid}}$.

We denote by $|\A|$ the number of states in $\A$. An \ra $\A$ is \emph{state-minimal} if $|\A| \le |\A'|$ for every equivalent \ra $\A'$, and \emph{data-minimal} if it is a $k$-\ra and every equivalent $\A'$ is a $k'$-\ra with $k' \ge k$. $\A$ is \emph{minimal} if it is both state- and data-minimal. A state-minimal or data-minimal \ra recognizing $\lang(\A)$ always exists, but a minimal \ra (even a \dra) need not. Consider the following example.

\begin{example}\label{ex:inc_dec_sequences}
For each $n \in \mathbb{N}$, let $L_n$ be the language of words $w$ over $\alphabetQ$ that form a strictly increasing or decreasing sequence of length $n$.
For example, when $n=3$, the word $w = 2\cdot \text{-}1\cdot \text{-}3.5$ belongs to $L_n$, whereas $u = 1\cdot 2$ and $v = 2\cdot \text{-} 3\cdot 0$ do not.
Consider a $1$-\dra $\A_n$ that recognizes $L_n$.
It has two fashions---increasing and decreasing---and always stores the most recent input in its single register.
Given $w = a_1\dots a_m$, $\A_n$ enters the increasing (resp. decreasing) fashion if $a_2 > a_1$ (resp. $a_2 < a_1$).
For $2 < i < n$, it remains in the current fashion provided $a_i$ is larger (resp. smaller) than the value stored in the register; otherwise it rejects.
Length $m = n$ is verified using control states.
Since both fashions require $n-1$ intermediate states (plus initial and final states), $\A_n$ has $2n$ states.

While $\A_n$ is data-minimal for $L_n$, it is not state-minimal since an equivalent $(n-1)$-\dra $\A'_n$ with $|\A'_n| < |\A_n|$ exists.
This \dra stores the first $n-1$ symbols of the input in its registers and checks on the fly whether these values form a strictly increasing or strictly decreasing sequence.
It needs $n+2$ states: $n+1$ states to verify the word length and one rejecting state to make $\A'_n$ complete.
\qed
\end{example}



As shown in \Cref{ex:inc_dec_sequences}, a minimal DRA need not exist for every data language $L$. To address this, we introduce \emph{well-typeness}, which ensures the existence of a \emph{minimal well-typed} \dra for every {\dra}-recognizable language. Notably, the canonical DRAs of \cite{Ley:10:TR} also satisfy this condition. Restricting DRAs to be well-typed 
preserves expressive power, though it may increase the number of states.

\begin{definition}
Let $\A$ be an \ra over $(\Sigma, R)$. We call $\A$ \emph{well-typed}  if for every two configuration transitions $(p, u) \xrightarrow{a}_{\A} (q, v)$ and  $(p', u') \xrightarrow{a'}_{\A} (q, v')$ ending in the same state $q$, we have $v \sim_R v'$. 
\end{definition}

Intuitively, an \ra is well-typed if the $\sim_R$-type of its registers is uniquely determined by the current state. 
In the remainder of the paper, \emph{all DRAs} are assumed to be \emph{well-typed unless stated otherwise}.

\section{\dra{s} over Non-Dense Ordered Domains}
\label{sec:non-dense}



In this section, we will present the Myhill–Nerode theorem for DRAs over ordered domains, which forms the basis for learning and minimization. We then solve open problems for DRAs over \emph{non-dense} domains---most notably minimization and configuration reachability. Together, these results lay the foundation for a unified learning framework for DRAs over both dense and non-dense settings.

As with other state machines, minimization for \wdra{s} is tightly related to the Myhill–Nerode theorem, as shown in~\cite{Ley:10:TR}. We now introduce some notions needed to state this theorem for \wdra{s}.
For a language $L$ and two words $u,v$, a word $w$ is called an \emph{$L$-distinguishing extension} for $u$ and $v$ if exactly one of the words $u \cdot w$ and $v \cdot w$ belongs to $L$. For regular languages $L$, two words are Nerode-congruent when no $L$-distinguishing extension exists. 
For data languages, however, word classification additionally requires the notion of \emph{memorability} introduced in~\cite{Ley:10:TR}.
Consider the language $L_7$ in \Cref{ex:inc_dec_sequences}. Any \dra $\mathcal{A}$ recognizing $L_7$, after reading  $u = 2 \cdot 3 \cdot 6$, must store the last value $6$ in a register, since $\mathcal{A}$ must check whether the next input value is greater than $6$. Thus, $6$ is \emph{$L$-memorable} in $u$, whereas $2$ and $3$ are not. The formal definition is as follows:

\begin{definition}\label{def:memorable}
 Let $L$ be a language over $(\Sigma, R)$. 
 A symbol $a \in \Sigma$ is \emph{$L$-memorable} in a word $u \in \Sigma^*$ if there exist a word $u'$ of the same $\sim_R$-type as $u$, symbols $a', b' \in \Sigma$, a word $w \in \Sigma^{*}$ and an $R$-preserving mapping $\sigma$ with $\sigma(a') = b'$ and $\sigma(d) = d$ for all other $d \in u' \cdot w$, such that:
(1) $u \cdot a \sim_R u' \cdot a' \sim_R \sigma(u' \cdot a')$; and (2) $w$ is an $L$-distinguishing extension for $u'$ and $\sigma(u')$.
\end{definition}

Intuitively, only memorable symbols determine the membership of future extended words. If $\alphabet$ is dense, we can let $u' = u$, and replace $a$ in $u$ with a nearby symbol $b$, obtaining a word $\sigma(u) $ with $\sigma(u) \sim_R u$ and check whether this change affects the membership of its extensions in $L$, i.e., there is an $L$-distinguishing extension $w$ for $u$ and $\sigma(u)$.
In the non-dense setting, such a replacement $b$ may not exist. We therefore need to construct a word $u'$ of the same type whose data values have sufficient space to allow a valid replacement $b'$ and then test on $u'$. 

We write $mem_L(u) = a_1 \dots a_k$ for the sequence of $L$-memorable symbols in $u$, ordered so that for every $a_i, a_j$, $j > i$ if the last occurrence of $a_j$ in $u$ occurs after the last occurrence of $a_i$. 
Let $\mathcal{A}$ be a \dra recognizing $L$. 
If $(q_0, \emptyword) \xrightarrow{u}_{\mathcal{A}} (q, v)$ for some configuration $(q, v)$, then by definition, $mem_L(u)$ is a subsequence of $v$~\cite{Ley:10:TR}.
We are now ready to define the Nerode congruence for data languages.

\begin{definition}\label{def:eq-relation-data-languages}\emph{(cf.~\cite{Ley:10:TR})}
Given a data language $L$ over $(\Sigma, R)$, we define an equivalence relation $\cong_L$ on $\Sigma^*$, where for every $u,v\in\Sigma^*$, we write $u \cong_L v$ if:
\begin{itemize}
\item $mem_L(u)
\sim_R mem_L(v)$, and
\item for all $x,y \in \Sigma^*$, if $mem_L(u)\cdot x \sim_R mem_L(v) \cdot y$, then $u\cdot x \in L \Leftrightarrow v\cdot y \in L$.
\end{itemize}
\end{definition}

Intuitively, equivalent words 
$u$ and $v$ must have memorable words of the same type. Furthermore, when their memorable-word extensions share the same type, extending $u$ and $v$ by the corresponding suffixes yields the same membership in $L$, since only memorable words affect future membership.

Let $L$ be a data language over $(\Sigma, R)$ and $k = \max \left\{|mem_L(u) \mid u \in \finwords \right\}$. 
Let $[u]_{\cong_L}$ denote the equivalence class defined by $\cong_L$ where $u\in\finwords$ belongs to. 
We define the \emph{canonical $k$-RA} 
$\mathcal{C}_L = (Q, q_0, F, \Delta)$ for $L$ as follows:
\begin{itemize}
    \item $Q = \bigcup_{i=0}^k Q_i$ where $Q_i = \left\{[u]_{\cong_L} \subseteq \finwords \mid u \in \finwords, |mem_L(u)| = i\right\}$, 
    \item $q_0 = [\emptyword]_{\cong_L}$, $F = \left\{[u]_{\cong_L} \in Q \mid u \in L\right\}$, and
    \item $([u]_{\cong_L},\tau,E, [u']_{\cong_L}) \in \Delta$ if there exists a symbol $a \in \Sigma$ such that:
    \begin{itemize}
        \item $u \cdot a \cong_L u'$, $mem_L(u)\cdot a$ has type $\tau$, and
        \item $E$ is the set of indices $i$ in the sequence
        $a_1 \dots a_d = mem_L(u) \cdot a$ such that either 
        $a_i \notin mem_L(u\cdot a)$ or $a_i = a$ for $i < d$.
    \end{itemize}
\end{itemize}
In general, $\mathcal{C}_L$ may have infinitely many states or registers and may be nondeterministic. 
When $L$ is \dra-recognizable, $\mathcal{C}_L$ is exactly the minimal \wdra for $L$.  
This Myhill–Nerode theorem for \dra{s} was established in~\cite{Ley:10:TR}:

\begin{theorem}\label{thm:myhill}\emph{(cf.~\cite{Ley:10:TR})}
Let $L$ be a data language. Then $L$ is recognized by a \dra iff $\cong_L$ has finite index. 
Moreover, every \dra-recognizable language $L$ has a canonical automaton $\mathcal{C}_L$ which  is the unique minimal \wdra up to isomorphism.
\end{theorem}

With \Cref{thm:myhill}, the standard DFA minimization algorithm extends to \wdra{s}. Given a \wdra\ $\mathcal{A}$, we merge states $p$ and $q$ whenever the residual languages of the configurations $(p,u)$ and $(q,v)$, for some $u,v$, coincide.
The existence of such $u$ and $v$ also ensures that $p$ and $q$ are reachable from the initial configuration, with $u$ and $v$ serving as witnesses. Checking this language equivalence reduces to solving the \emph{configuration reachability problem} in the product automaton.
Consequently, the minimization algorithm\footnote{
 Complete reductions and pseudocode for minimizing \wdra{s} are provided in App.~\ref{app:min_alg}.} depends on solving 
the \textbf{configuration reachability problem}, which is interesting in its own right:
Given a configuration $(p, u)$ and a state $q$, determine whether there exists a word $w_q$ such that a run from $(p, u)$ reaches some configuration $(q,w)$ on input word $w_q$.

As shown in~\cite{Ley:10:TR}, the configuration reachability problem is trivially decidable for dense domains.
Given a configuration $(p, u)$ and a transition $(p, \tau, E, q)$, the density property ensures that there always exists a value $a$ such that $u\cdot a$ has type $\tau$.
Hence, with a complete \dra, checking whether a configuration $(p, u)$ can reach a state $q$ reduces to deciding whether there exists a sequence of transitions from $(p, u)$ that leads to the state $q$.
This can be checked easily.

The decidability of the configuration reachability problem in the non-dense setting, however, remains \emph{open} since the above approach fails.
Consider the \dra $\mathcal{A}_{mid}$ over $(\alphabetZ, <)$ in Fig.~\ref{fig:DRA_mid}: we have $(q_2, u) \xrightarrow{a} (q_3, \emptyword)$ when $u = 1\cdot 5$ and $a = 3$, but for $u = 1\cdot 2$, no value of $a$ produces such a run.
To address this gap, we provide the \emph{first}  decidable complexity result of this problem over non-dense domains.
\begin{restatable}{theorem}{thmReachNonDense}\label{thm:reach}
The configuration reachability problem for \dra{s} over $(\mathbb{Z},<)$ is decidable.
\end{restatable}

As an immediate result of \Cref{thm:reach}, we obtain the following:

\begin{corollary}\label{coro:minimization-non-dense-decidable}
The minimization problem for \wdra{s} over $(\mathbb{Z},<)$ is decidable.
\end{corollary}


In fact, a simpler route to achieve \Cref{coro:minimization-non-dense-decidable} is available. The next theorem shows that minimization and language equivalence for \wdra{s} over both dense and non-dense domains fit into a single unified framework.
That is, we can reduce these problems over non-dense domains to the problems over dense domains which are already known to be decidable \cite{Ley:10:TR}. However, this unification does \emph{not} extend to the configuration reachability problem.

\vspace*{-1mm}
\begin{theorem}\label{thm:Z_eq_Q}
Let $\mathcal{A}$ be a DRA, and let $L_{\mathbb{Q}}$ and $L_{\mathbb{Z}}$ be the data languages recognized by $\mathcal{A}$ over $(\mathbb{Q},<)$ and $(\mathbb{Z},<)$, respectively. For any $w \in \mathbb{Q}^*$ and $w' \in \mathbb{Z}^*$ with $w \sim_{<} w'$, it holds that $w \in L_{\mathbb{Q}}$ iff $w' \in L_{\mathbb{Z}}$. Furthermore, a DRA over $(\mathbb{Q},<)$ is the minimal \wdra recognizing $L_{\mathbb{Q}}$ iff the same DRA over $(\mathbb{Z},<)$ is the minimal \wdra recognizing $L_{\mathbb{Z}}$. This correspondence remains valid when either $\mathbb{Q}$ or $\mathbb{Z}$ is replaced by $\mathbb{R}$.
\end{theorem}

\begin{proof}
Given a \dra $\mathcal{A}$, we use $\mathcal{A}_{\mathbb{Q}}$ and $\mathcal{A}_{\mathbb{Z}}$ to represent the \dra{s} $\mathcal{A}$ over $(\mathbb{Q},<)$ and $(\mathbb{Z},<)$, respectively. 
For every $w \in \mathbb{Q}^*$, there exists $w' \in \mathbb{Z}^*$ s.t. $w' \sim_< w$, and vice versa. Therefore, $L(\mathcal{A}_{\mathbb{Z}}) = L(\mathcal{A}_{\mathbb{Q}}) \cap \mathbb{Z}^*$ and $L(\mathcal{A}_{\mathbb{Q}}) = \{ w \mid \exists w' \in L(\mathcal{A}_{\mathbb{Z}}) \wedge w' \sim_< w\}$.
This implies that $\mathcal{A}_{\mathbb{Q}}$ is the minimal \wdra recognizing $L(\mathcal{A}_{\mathbb{Q}})$ if and only if  $\mathcal{A}_{\mathbb{Z}}$ is the the minimal \wdra recognizing $L(\mathcal{A}_{\mathbb{Z}})$.
The argument remains valid if the ordered domain $\mathbb{Q}$ or $\mathbb{Z}$ is replaced by $\mathbb{R}$.
\qed
\end{proof}

\vspace*{-2mm}
\section{Active Learning of \dra{s}}
\label{sec:learning}

In this section, we present an algorithm for learning a DRA of an unknown data language $L$ using \emph{polynomially} many queries in the active learning framework.
In this framework, a \emph{learner} aims to infer an automaton recognizing $L$, while a \emph{teacher} knows $L$ and answers \emph{membership} queries $\MQ(u)$ (checking whether $u \in L$) and \emph{equivalence} queries $\EQ(\hypothesis)$ (checking whether $\hypothesis$ recognizes $L$).

The learner begins by issuing membership queries to populate an \emph{observation table} $\mathcal{T}$, from which it constructs a conjectured automaton $\hypothesis$.
An equivalence query then checks its correctness: if the teacher replies “Yes”, learning terminates; otherwise a counterexample $w \in L \ominus \lang(\hypothesis)$ is returned and used to refine $\hypothesis$.
This process repeats until a correct automaton is found.
While originally developed for DFAs, we adapt this framework to DRAs over ordered domains.
By \Cref{thm:Z_eq_Q}, it suffices to learn DRAs over $(\alphabetQ,<)$, since the case of $(\alphabetZ,<)$ reduces to it; See~\Cref{app:non-dense-learning} for more details on the reduction.

Unlike DFA learning, data languages lack known polynomial-time learning algorithms with respect to the number of queries, even for only equality tests on data~\cite{DBLP:conf/vmcai/HowarSJC12}.
This barrier can be overcome by using also \emph{memorability queries}~\cite{DBLP:conf/mfcs/BalachanderFGT25}, $\Mem(u)$, which return the memorable word $mem_L(u)$ of a word $u$.

We first present a DRA learner using membership, equivalence, and memorability queries, and then give improved complexity results for some decision problems of DRAs that are closely related to solving queries.


\subsection{Learning DRAs with Memorability Queries} \label{sec:learning-mem}


We fix a target DRA-recognizable language $L$ over $(\alphabet, R) = (\alphabetQ, <)$ in the whole section.
The learner is assumed to know $(\alphabet, R)$ and maintains an observation table to record its knowledge of $L$.

\begin{figure}[t]
    \centering
\includegraphics[scale=0.8]{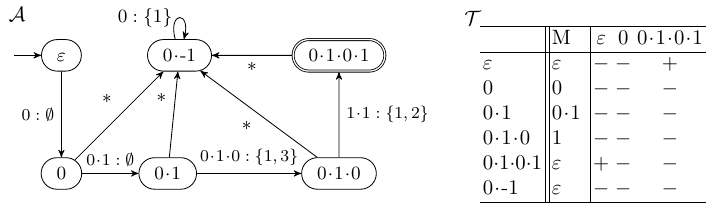}
    \caption{The example DRA $\target$ where the missing transitions are labeled with $*$, and the observation table $\obtable$ where the extension set $X$ are omitted.}
\label{fig:table-and-dra}
\end{figure}

\begin{definition}[Observation Table]\label{def:obtable}
%
An observation table $\obtable$ for a DRA over $(\Sigma, R)$ is a tuple $(U, X, S, M, f)$, where $U$, $X$, and $S$ are finite sets of data words: $U$ is the set of \emph{prefixes}, $X$ the \emph{extensions}, and $S$ the \emph{suffixes}.
The function $M: U \cup X \to \finwords$ records the memorable word of each word $u\in U\cup X$, i.e., $M(u) = mem_L(u)$.
The classification function $f : U \times S \to \{+, -\}$ labels each $u \cdot w$ as in $L$ or not, i.e, $f(u, w) = \MQ(u\cdot w)$.
We require (1) $U \cup X$ to be prefix-closed and (2) $\emptyword \in U$ and $\emptyword \in S$.
\end{definition}
Table $\obtable$ in Fig.~\ref{fig:table-and-dra} shows a final observation table for 
$L_{\mathit{rep}} = \{xyxy \in \alphabetQ^4 \mid x < y \}$.
The rows correspond to the prefix set
$U = \{\emptyword, 0, 0\mdot 1, 0\mdot 1\mdot 0, 0\mdot 1\mdot 0\mdot 1, 0\mdot \text{-}1\}$,
and the columns contain the memorable function $M$ and the suffix set
$S = \{\emptyword, 0, 0\mdot 1\mdot 0\mdot 1\}$.
The set $X$ is omitted here, as it is only used internally when constructing conjectured DRAs.
From this table, we can obtain the minimal DRA $\target$ (with respect to $\cong_L$), also shown in Fig.~\ref{fig:table-and-dra}.
Each row $u \in U$ naturally corresponds to a state of $\target$: $u$ serves as the representative of its $\cong_L$-equivalence class, and $\target$ reaches the associated state exactly after reading $u$.
For instance, the empty word $\emptyword$ is the initial state, whereas $0\mdot 1\mdot 0\mdot 1$ is the only final state.

In the remainder of the paper, we abusively use the word representatives in $U$ as state names.
While constructing the observation table, as in the classic $L^*$ algorithm for DFAs, we ensure that for every state $u \in U$, each one-letter extension $u \cdot b$ appears in the table, i.e., in $X$.  
To compute the one-letter extensions of $u$ (i.e., $u \cdot b$ for $b \in \alphabet$), we observe that  
(i) only the memorable word $M(u)$ is relevant, since the types of future words depend solely on its memorable letters; and  
(ii) we do not need to enumerate all $b \in \alphabet$, but only a finite set that covers all possible types of $M(u)\cdot b$.

Let $a_1 < \cdots < a_d$ be the ordered sequence of \emph{distinct} letters in $M(u)$. Then $b$ can take one of the four types of values: 
\begin{enumerate}
    \item $b \in M(u)$, i.e., it means that $b$ occurs also in $M(u)$;
    \item $b = a_1 -1$, i.e, we have that $b < a_i$ for all $1\leq i \leq d$;
    \item $b = a_d + 1$, i.e., it gives that $b > a_i$ for all $1\leq i \leq d$; and finally
    \item $b = (a_i + a_{i+1})/2$ for some $1\leq i<d$, i.e., $a_j <b < a_k$ for all $j \le i$ and $k \ge i+1$.
\end{enumerate}

Note that the fourth type of values will be ignored if $|M(u)| = 1$ and we set $b = 0$ when $u = \emptyword$.
If $R$ is the identity test, it suffices to add values $b \in M(u)$ or a fresh letter $b \notin M(u)$.
For example, if $M(u) = 0\mdot 1$, then $b \in \{0, 1, 2\}$.

To ensure that $U$ matches the state set of the target DRA, we must determine when two words $u_1, u_2 \in U$ are \emph{nonequivalent} under the Nerode congruence $\cong_L$ (\Cref{def:eq-relation-data-languages}).
Recall that $u_1 \not\cong_L u_2$ holds precisely when
(i) their memorable word types differ, i.e., $mem_L(u_1) \not\sim_R mem_L(u_2)$, or
(ii) there exist $x, y \in \finwords$ such that
$mem_L(u_1)\cdot x \sim_R mem_L(u_2)\cdot y$ but
$u_1 \cdot x \in L \not\Leftrightarrow u_2 \cdot y \in L$.
While testing whether two words have the same type is straightforward, finding witness words $x$ and $y$ is more challenging.
In contrast, the classical $L^*$ algorithm for DFAs~\cite{DBLP:journals/iandc/Angluin87} needs only a \emph{single} witness to distinguish two words.

Fortunately, since $\alphabet$ is dense, it suffices to find an order-preserving mapping $\sigma$ with $\sigma(mem_L(u_1)) = mem_L(u_2)$ and a \emph{single} $L$-distinguishing suffix $w$ that separates $\sigma(u_1)$ from $u_2$.
The following lemma justifies this reduction:

\begin{restatable}{lemma}{witnessNotEq}\label{lem:witness_of_not_eq}
Let $\alphabet$ be a dense set. For two words $u_1, u_2 \in \Sigma^*$ with $mem_L(u_1) \sim_R mem_L(u_2)$, $u_1 \not\cong_L u_2$ if and only if there exist a bijective order-preserving mapping $\sigma$ and a word $w$ such that $\sigma(mem_L(u_1)) = mem_L(u_2)$ and $ \sigma(u_1) \cdot w \in L \not \Leftrightarrow u_2 \cdot w \in L$.
\end{restatable}

A further challenge is that infinitely many order-preserving mappings $\sigma$ may exist, making enumeration impossible.
Nevertheless, we prove that it suffices to find \emph{one} suitable mapping.
Notably, such a bijective order-preserving mapping $\sigma$ with $\sigma(mem_L(u_1)) = mem_L(u_2)$ can be obtained easily, as described in~Sect.~\ref{sec:pre}.

\begin{restatable}{lemma}{inequivalenceS}\label{lem:inequivalence-S}
Let $\alphabet$ be a dense set, and let $u_1, u_2$ be distinct words. 
If $\sigma$ is a bijective order-preserving mapping with $\sigma(mem_L(u_1)) = mem_L(u_2)$, then
$u_1 \not\cong_L u_2$ iff there exists $w \in \finwords$ such that
$\sigma(u_1)\cdot w \in L \not\Leftrightarrow u_2 \cdot w \in L$.
\end{restatable}

To derive a conjectured DRA from the observation table $\obtable$, the table must be \emph{closed}. 
We introduce the notion of \emph{S-consistency}: a word $v \in X$ is said to be \scap a prefix $u \in U$ if $M(v) \sim_R M(u)$ and, for every $w \in S$, we have $\MQ(\sigma(v)\cdot w) = f(u, w)$, where $\sigma$ is an order-preserving mapping on $\Sigma$ s.t. $\sigma(M(v))=M(u)$. 
By \Cref{lem:inequivalence-S}, the table $\obtable$ is said to be closed if, for each extension $ua \in X$ of a word $u \in U$, there exists a prefix $u' \in U$ s.t. $ua$ is \sca to $u'$.
Intuitively, the target DRA $\target$ moves to state $u'$ at state $u$ when reading $a \in \alphabet$ if $ua$ is \scap $u'$.
Note that \scn is not an equivalence relation, as it does not need to be symmetric nor transitive.

Our algorithm relies on two key properties of \scn:
(i) if $u$ is not \scap $v$, then $u \not\cong_L v$;
(ii) if $u$ is not \scap $v$, then $u$ is not \spcap $v$ whenever $S \subseteq S'$.

\begin{figure}[t]
    \centering
    \includegraphics[scale=0.8]{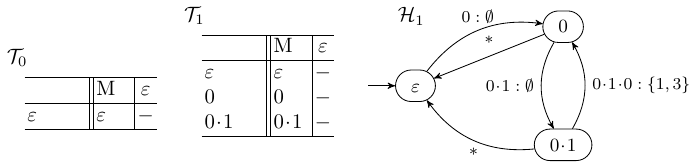}
    \caption{Unclosed table $\obtable_0$, closed table $\obtable_1$ and its induced DRAs $\hypothesis_1$.}
    \label{fig:running-example-tables-dras}
\end{figure}

If $\obtable$ is not closed, the learner invokes 
\textsf{CloseTable}$(\obtable,\MQ,\Mem)$, which repeatedly:
(i) finds a word $ua \in X$ with $u \in U$ that has no \sca $u' \in U$;
(ii) adds $ua$ to $U$ and adds also all its one-letter extensions $ua b$ to $X$;
(iii) fills the new row labeled with $ua$ by setting $f(ua, w) = \MQ(ua \cdot w)$ for all $w \in S$ and updates $M(ua b)=\Mem(uab)$ for all new extensions $uab \in X$.

    Consider the initial table $\obtable_0$ in Fig.~\ref{fig:running-example-tables-dras} for learning
    $L = \{xyxy \in \alphabetQ^4 \mid x < y\}$ with $X = \{0\}$ and $M(0) = 0$.
    To close the table, the learner checks whether $\emptyword \cdot 0$ is \scap $\emptyword$.
    Since $M(\emptyword \cdot 0) = 0$ differs in word type from $M(\emptyword) = \emptyword$, the check fails, and $0$ is added to $U$.
    We update $f$ with $f(0, \emptyword) = \MQ(0 \cdot \emptyword) = -$ for the new row $0$ and column $\emptyword$.
    Further, the one-letter extensions of $0$ are $\{0\cdot 0, 0\cdot \text{-}1, 0\cdot 1\}$, so $X$ becomes
    $\{0, 0\cdot 0, 0\cdot \text{-}1, 0\cdot 1\}$.
    Using memorability queries, we also update $M$ as $M(0\mdot 0) = \emptyword$, $M(0\mdot \text{-}1) = \emptyword$, $M(0\mdot 1) = 0\mdot 1$.
    Continuing in this way, the learner adds $0\mdot 1$ to $U$ and obtains the closed table $\obtable_1$ in Fig.~\ref{fig:running-example-tables-dras}.

With all extensions computed for each new state, the conjectured DRA $\hypothesis$ we construct will be complete.
Once $\obtable$ is closed, we can induce $\hypothesis$ using~\Cref{def:hypothesis-construction}.
\begin{definition}[DRA Construction]\label{def:hypothesis-construction}
    Let $\obtable = (U, X,S, M,  f)$ be a closed observation table 
    and $k = \max\left\{|M(u)| \mid u \in U\right\}$.
    We construct a $k$-DRA $\hypothesis=(Q, q_0, F, \Delta)$ from $\obtable$ as follows:
    \begin{itemize}
    \item $Q = \bigcup_{i=0}^k Q_i$ where $Q_i = \left\{ u \in U \mid |\Mem(u) | = i\right\}$, $q_0 = \emptyword \in Q_0$, $F = \left\{u \in U \mid f(u, \emptyword) = +\right\}$, and
    \item $(u, \tau, E, u') \in \Delta$ if there is an extension $u \mdot a \in X$ s.t. (1) $u' \in U$ is the first found word with which $u\mdot a$ is \sca, (2) $\tau = M(u)\cdot a = a_1 \cdots a_d$, and $E$ contains indices $i$ where $a_i \notin M(u\cdot a)$ or $a_i = a_d = a$ with $i < d$. 
    \end{itemize}
\end{definition}
It immediately follows from~\Cref{thm:myhill} that:
\begin{lemma}\label{lem:correctness-dra-construction}
If $U$ contains exactly the set of word representatives defined by $\cong_L$, then $\hypothesis$ is a minimal DRA recognizing the target language $L$.
\end{lemma}

Consider constructing the DRA $\hypothesis_1$ from the closed table $\obtable_1$ in Fig.~\ref{fig:running-example-tables-dras}.
Since the maximal length of a memorable word is $|0\mdot 1| = 2$, $\hypothesis_1$ is a $2$-DRA. The set of states is $Q = \{\emptyword\} \uplus \{0\} \uplus \{0\mdot 1\}$, with $q_0 = \emptyword$, and $F = \emptyset$ because no $u \in U$ satisfies $f(u, \emptyword) = +$.
Here, we focus on the transition from $0$ to $0\mdot 1$ in $\Delta$.
The one-letter extension $0\mdot 1$ of state $0$ is \scap $0\mdot 1$, so reading letter $1$ at $0$ leads to state $0\mdot 1$. We then set $\tau = M(0)\cdot 1 = 0\mdot 1$ and $E = \emptyset$, since both $0$ and $1$ belong to $M(0\mdot 1)$ and no letters need to be removed.

\begin{algorithm}[t]
\caption{Function $\textsf{cexAnalyze}(\obtable, \hypothesis, w, \MQ(\cdot))$ that analyses $w$ and returns a suffix $v$. The prerequisite is that $w \in L\ominus \lang(\hypothesis)$.}
\label{algo:cex-analysis}
    $u  \gets \emptyword, r \gets \emptyword$, $\textit{mq} \gets \MQ(w)$\;
    \While{True}{
        
        Let $a \mdot z := w$ where $a \in \alphabet$; \CommentSty{ //$w = \emptyword$ is impossible}\; 
        \CommentSty{//INVARIANT: $r = M(u)$}\;
        Let $(u, r) \xrightarrow{M(u)\cdot b: E} (u', r')$ be the transition taken in $\hypothesis$ when reading $a$\;
        \CommentSty{//INVARIANT: $r\cdot a = M(u)\cdot a\sim_R M(u)\cdot b$, $M(u') \sim_R M(ub) = r'$}\;
        Let $\sigma_{1}$ be an order-preserving mapping from $M(u)\cdot a$ to $M(u)\cdot b $\;
        \CommentSty{//Since $r = M(u)$, $r' = M(u\cdot b)$ is obtained by removing letters in positions $E$ from $r\cdot b$}\;
        Let $\sigma_2$ be an order-preserving mapping from $M(u\cdot b)$ to $M(u')$\;
        \If{$\MQ(u'\mdot \sigma_2(\sigma_{1}(z))) \not = \textit{mq}$}{
            \Return new suffix $\sigma_2(\sigma_{1}(z))$\;
        }
        
        $u \gets u', r \gets M(u'), w\gets \sigma_2(\sigma_{1}(z));$\CommentSty{   //$\MQ(u'\cdot \sigma_2(\sigma_1(z))) = \textit{mq}$} \;
    }
\end{algorithm}

Once the conjectured DRA $\hypothesis$ is constructed, the learner poses an equivalence query $\EQ(\hypothesis)$ to the teacher.
If the answer is positive, $\hypothesis$ is returned as the result. Otherwise, the teacher provides a counterexample $w \in \lang(\hypothesis)\ominus L$.
The learner analyzes $w$ to update the observation table, allowing it to correct $\hypothesis$ by adding new states or redirecting transitions.
The counterexample analysis procedure is described in \Cref{algo:cex-analysis}.
An example of counterexample analysis can be found in \Cref{app:examples}.

In \Cref{algo:cex-analysis}, the counterexample $w$ is processed letter by letter. In each iteration, we assume the current configuration is $(u, r)$ and $\MQ(u\cdot a z) = \textit{mq}$, where $w = a\cdot z$.
Reading letter $a$ at state $u$, $\hypothesis$ transitions to $u'$. However, this transition may have been constructed using a different letter $b$ such that $ub$ is \scap $u'$ and $a \neq b$, where $M(u)\cdot a \sim_R M(u)\cdot b$.

Our goal is to identify an extension $ub$ wrongly classified as equivalent to $u'$. We first normalize the remaining suffix $z$ for $ub$ via the mapping $\sigma_1$ from $M(u)\cdot a$ to $M(u)\cdot b$, preserving membership results: $\MQ(ub\cdot \sigma_1(z)) = \textit{mq} = \MQ(ua\cdot z)$, since both runs visit the same state sequence in the target DRA of $L$.
To verify whether $ub \cong_L u'$, we further normalize $\sigma_1(z)$ for $u'$ using $\sigma_2$ from $M(ub)$ to $M(u')$, yielding $v = \sigma_2(\sigma_1(z))$. If $\MQ(u'\cdot v) \neq \textit{mq}$, then $u' \not\cong_L ub$, according to \Cref{lem:inequivalence-S} since $\sigma_2(ub\cdot \sigma_1(z)) \sim_R \sigma_2(ub)\cdot v$, giving $\MQ(\sigma_2(ub)\cdot v) = \textit{mq}$.
Otherwise, we proceed to the next iteration with $u \gets u'$, $r \gets M(u')$, and $w \gets v$. Since $w \in L \ominus \lang(\hypothesis)$, before reaching the last state $u''$ where $\MQ(u''\cdot \emptyword) \neq \textit{mq}$, the membership results must change at some point. Hence, the \textbf{if} condition will eventually be satisfied, and a suffix $v$ will be returned.

The correctness of \Cref{algo:cex-analysis} is guaranteed by the following lemma.
\begin{restatable}{lemma}{correctnessCEXanalysis}\label{lem:correct-ness-cex-analysis} 
    Let $w$ be the counterexample provided by the teacher.
    \Cref{algo:cex-analysis} terminates after at most $|w|$ iterations and returns a suffix $v$ such that $ub \in X$ is not $S  {\uplus} \{v\}$-consistent with a word $u' \in U$ with which $ub$ is currently $S$-consistent.
\end{restatable}


\begin{algorithm}[t]
\caption{The DRA learner}\label{algo:the-algorithm}
    Initialize table $\obtable = ( U, X, S, M, f)$ where $U = S = \{\emptyword\}$, $X = \{0\}$, $M(\emptyword) = \emptyword$, $M(0) = \Mem(0)$, and $f(\emptyword, \emptyword) = \MQ(\emptyword)$\;
    $\textsf{CloseTable}(\obtable, \MQ(\cdot), \Mem(\cdot))$ and let $\hypothesis$ be the DRA constructed from $\obtable$\;
    Let $(\textit{res}, w)$ be the response from the teacher on $\EQ(\hypothesis)$\;
    \While{$\textit{res} = {No}$}{
        $v \gets \textsf{cexAnalyze}(\obtable, \hypothesis, w, \MQ(\cdot))$\;
        Add $v$ to the suffix set $S$ in $\obtable$ and set $f(u, v) = \MQ(u\mdot v)$ for all $u \in U $\;
         $\textsf{CloseTable}(\obtable, \MQ(\cdot), \Mem(\cdot))$ and let $\hypothesis$ be the DRA constructed from $\obtable$\;
        Let $(\textit{res}, w)$ be the response from the teacher on $\EQ(\hypothesis)$\;
    }
    \Return $\hypothesis$\;
\end{algorithm}


We now formally define the learner in \Cref{algo:the-algorithm}.
To initialize the observation table $\obtable$, the learner inserts the empty word $\emptyword$ into $U$. 
The extension set is initialized as $X = \{0\}$, since $0$ has the same type as all one-letter words.
It fills entries via membership queries $\MQ(\cdot)$ and memorability queries $\Mem(\cdot)$, so $M(\emptyword) = \emptyword$ and $M(0) = \Mem(0)$. 
The function $\textsf{CloseTable}$ closes $\obtable$ with the help of $\MQ(\cdot)$ and $\Mem(\cdot)$. Conjectured DRAs $\hypothesis$ are then derived from $\obtable$ following \Cref{def:hypothesis-construction}.
On receiving a counterexample $w$, $\textsf{cexAnalyze}$ (\Cref{algo:cex-analysis}) returns a suffix $v$ that is added to $S$. New entries are filled via $\MQ(\cdot)$. The refinement loop of $\hypothesis$ terminates once the teacher returns a positive answer.

The soundness and completeness of \Cref{algo:the-algorithm} directly follow from \Cref{lem:inequivalence-S,lem:correctness-dra-construction,lem:correct-ness-cex-analysis}.

\begin{restatable}{theorem}{correctnessLearning}\label{thm:correctness-learning-algorithm}
    Let $\target$ be a complete canonical minimal \wdra of $L$ with $n$ states and $m$ transitions defined under $\cong_L$, and let $d$ be the maximum length of any counterexample returned by the teacher. It holds that:
\begin{itemize}
\item \Cref{algo:the-algorithm} terminates and returns a DRA $\hypothesis$ that is isomorphic to $\A$. 
\item \Cref{algo:the-algorithm} terminates after at most $m \times (n-1)$ equivalence queries.
\item \Cref{algo:the-algorithm} requires $m $ memorability queries and $\mathcal{O}(d \times m \times n + m^3 \times n^3)$ membership queries.
\end{itemize}
\end{restatable}

The returned DRA $\hypothesis$ recognizes $L$, as confirmed by the teacher.
\Cref{algo:the-algorithm} terminates because each counterexample either corrects a transition or adds a new state (\Cref{lem:correct-ness-cex-analysis}), 
and the number of equivalence classes under $\cong_L$ and transitions in target DRA $\target$ is finite.
Now we show that our algorithm runs in polynomial time in the number of queries.

In DFA learning, two words $u_1$ and $u_2$ are considered equivalent with respect to a set of suffixes $S$ if 
$\MQ(u_1 w) = \MQ(u_2 w)$ for all $w \in S$. 
Thus, when a transition $ub$ is misclassified as consistent with some $u'$, it can safely be added to~$U$.
In our setting, however, \scn does not guarantee this property, as it is not an equivalence relation.  
When \Cref{algo:cex-analysis} returns a distinguishing suffix $v$ separating $ub$ from the state $u'$ with which $ub$ is currently $S$-consistent, 
we add $v$ to the suffix set $S$ to correct the transition from $u$ over $b$.  
Consequently, $ub$ becomes not $S \uplus \{v\}$-consistent with $u'$.  
Each transition can be corrected at most $n-1$ times; 
that is, the target state of a transition over $b$ from $u$ may be misclassified no more than $n-1$ times, 
as once $ub$ is inconsistent with a state $u''$, it remains inconsistent with $u''$ thereafter.  
Hence, the learner asks at most $m \cdot (n-1)$ equivalence queries in total.
In practice, new states are typically discovered immediately after adding the suffix returned by the counterexample analysis.
The reasoning for membership and memorability queries can be found in \Cref{app:learning}.

\subsection{Complexity of Decision Problems}\label{sec:decision-problems}

This section studies the complexity of some key decision problems for \dra{s} over ordered data domains.
Understanding these complexities is essential, as they closely correspond to the queries required by our active learning algorithm, particularly when the teacher has the target data language in the form of a \dra.
All results in this section apply to \emph{both} well-typed and non-well-typed \dra{s}.


Let $\mathcal{A}_1 = (Q^1, q_0^1,F^1,\Delta^1)$ be a $k_1$-DRA  and $\mathcal{A}_2 = (Q^2, q_0^2,F^2,\Delta^2)$ be a $k_2$-DRA, respectively. 
The \emph{language intersection non-emptiness} asks whether there is a word $w \in L(\mathcal{A}_1) \cap L(\mathcal{A}_2)$. 
The \emph{language inclusion} problem asks whether $L(\mathcal{A}_1)\subseteq L(\mathcal{A}_2)$,  
equivalently whether $L(\mathcal{A}_1)\cap \overline{L(\mathcal{A}_2)}=\emptyset$,  
where $\overline{L(\mathcal{A}_2)}$ is obtained by complementing $\mathcal{A}_2$ via swapping accepting and non-accepting states since $\mathcal{A}_2$ is deterministic. 
\begin{restatable}{theorem}{thmInclusionPSPACE}\label{thm:pspace-inclusion}
For DRAs over $(\Sigma, =)$ and over $(\Sigma, <)$, the language intersection non-emptiness problem is PSPACE-complete. Consequently, the language inclusion problem is also PSPACE-complete.
\end{restatable}

{\color{black}
For the PSPACE upper bound, intersection non-emptiness is reduced to reachability in the product of the two DRAs, where it suffices to distinguish register valuations by their \emph{relative} data types. If an accepting run exists, there is one of length at most $|Q^1|\cdot |Q^2|\cdot (k_1+k_2)!$, so it can be guessed symbol by symbol while storing only the current product configuration and a counter. PSPACE-hardness already holds for DRAs with identity tests, via a reduction from linear-bounded deterministic Turing machines.
}

Note that by \Cref{thm:Z_eq_Q}, the above result holds for both dense and non-dense domains.

Given a DRA $\mathcal{A}$ over $(\Sigma, R)$ and a configuration $(q, v)$ of $\mathcal{A}$, we define the language of $\mathcal{A}^{(q, v)}$ as $L(\mathcal{A}^{(q, v)}) = \{ w \in \Sigma^* \mid \exists\, q_f \in F, u \in \Sigma^*.~(q, v) \xrightarrow{w}_\mathcal{A} (q_f, u) \}$.
Let $(q_1, v_1)$ and $(q_2, v_2)$ be two configurations of $\mathcal{A}_1$ and $\mathcal{A}_2$ respectively. 
The \emph{configuration equivalence problem} asks whether $L(\mathcal{A}_1^{(q_1, v_1)}) = L(\mathcal{A}_2^{(q_2, v_2)})$.
The \emph{language equivalence problem} asks whether $L(\mathcal{A}_1) = L(\mathcal{A}_2)$. 
It is a special case of the configuration equivalence problem, that is, it is equivalent to the configuration equivalence problem whether $L(\mathcal{A}_1^{(q_0^1, \varepsilon)}) = L(\mathcal{A}_2^{(q_0^2, \varepsilon)})$ where $(q_0^1, \varepsilon)$ and $(q_0^2, \varepsilon)$ are the initial configurations of $\mathcal{A}_1$ and $\mathcal{A}_2$ respectively.

\begin{restatable}{theorem}{thmLangEquiv}\label{thm:equiv-pspace}
    Both the configuration equivalence problem and the language equivalence problem are in PSPACE for DRAs over dense domain $(\Sigma, <)$. 
\end{restatable}

The configuration equivalence problem is referred to as the residual equivalence problem in~\cite{DBLP:conf/mfcs/BalachanderFGT25}, and is shown to be in PTIME for well-typed DRAs over unordered (identity) data domains $(\Sigma, =)$~\cite[Theorem~8]{DBLP:conf/mfcs/BalachanderFGT25}.

Given a $k$-DRA $\mathcal{A}$ over $(\Sigma, <)$, a word $u \in \Sigma^{*}$ and $a \in u$, the \emph{memorability problem} asks whether $a$ is $L$-memorable in $u$, that is, whether $a \in mem_L(u)$.
We show this problem is in PSPACE, which improves upon the NEXPTIME upper-bound for the memorability problem established in~\cite[Proposition~3]{Ley:10:TR}.

\begin{restatable}{theorem}{thmMemPSPACE}\label{thm:mem-pspace}
    The memorability problem is in PSPACE for DRAs over dense domain $(\Sigma, <)$. 
\end{restatable}
{\color{black}Over a dense domain, memorability reduces to configuration equivalence: replacing $a$ by a fresh value while preserving the order type yields two configurations that are inequivalent exactly when $a$ is memorable.}

\paragraph*{Experiments.} We implemented the proposed active learning algorithm in Python\footnote{\url{
liyong31/RALearning.git}} and evaluated it on the running examples in this paper as well as randomly generated DRAs. The results confirm that the numbers of membership, equivalence, and memorability queries grow polynomially with the size of the target automaton, in line with the theoretical bounds, with membership queries dominating in practice. We also compared our approach with \ralib~\cite{DBLP:journals/fac/CasselHJS16}; \ralib~uses tree-based queries with black-box approximate equivalence checking, whereas our method uses word-based queries with white-box exact equivalence checking. The experiments show that both approaches are effective in their respective settings. See \Cref{app:expr} and \Cref{app:experiments} for more details.


\section{Conclusion and Discussion}
\label{sec:conclusion}
\modify{
In this paper, we studied \dra{s} over ordered domains. We showed that, with respect to minimization, \dra{s} over dense and non-dense domains can be treated within a unified framework. This yields decidability of the minimization problem for \dra{s} over non-dense domains, resolving a previously open question \cite{Ley:10:TR}.

Building on this unified framework, we developed and implemented an active learning procedure for \dra{s} that makes use of membership, equivalence, and memorability queries. Our experimental results demonstrate that the learning framework presented in this paper complements existing approaches for \dra{s} over ordered domains.
}


A natural question arises: can memorability queries be removed to obtain an algorithm similar to the classic $L^*$?
We believe this is possible by using the notion of \emph{abstract suffixes} introduced in~\cite{DBLP:conf/vmcai/HowarSJC12}.
Instead of treating the suffix $v$ returned by \Cref{algo:cex-analysis} as a concrete word, we treat it as an \emph{abstract} suffix representing all possible suffixes $z$ for a prefix $u$ over a set $D_{u,v}$ with $z \sim_R v $.
Let $v$ contain $\ell$ distinct letters. 
For each prefix $u = b_1 \dots b_m \in U$, let $b'_1 < \dots < b'_m$ be the sorted values in $u$, $b'_0 =b'_1-1$, and $b'_{m+1}=b'_m+1$.
Then, 
$D_{u,v} = 
\bigcup_{i=0}^{m}\bigcup_{j=1}^{\ell+1} \{b'_i + \frac{j}{\ell+1}\times (b'_{i+1} - b'_i)\}$.
The suffixes $z$ cover all possible extensions of $u$ that have the same word type as $v$, analogous to our previous computation for the one-letter extensions of the memorable word $M(u)$.
Since $M(u)$ is unknown, we have to consider all possible suffixes of the same word type to infer $M(u)$, instead of a single suffix $v$ from a counterexample.

Checking \scn and the corresponding counterexample analysis procedures also become more involved.
We leave a full learning algorithm using only membership and equivalence queries as future work.
Unsurprisingly, this approach requires an exponential number of membership queries, as also in~\cite{DBLP:conf/vmcai/HowarSJC12}, to fill the observation table since all concrete suffixes $z$ over $D_{u, v}$ must be enumerated and added to the suffix set for $u$.

{\color{black}We also note that the register automata underlying \ralib~operate on data words over $A \times D$, where $A$ is a finite action alphabet and $D$ is an infinite data domain, whereas our model considers words directly over $D$. Hence, the two models are not directly equivalent, since the action-labelled model can distinguish inputs carrying the same data value but different actions. A natural extension of the model considered here is to $A \times D$, with word types preserving the finite action sequence while comparisons and renamings apply only to the data components. We conjecture that this action-extended variant has the same expressive power as the corresponding action-labelled register automata. Establishing a formal correspondence between the two models, and determining whether our learning and complexity bounds carry over to this extension, are interesting directions for future work.}

\paragraph*{Acknowledgements.}
We would like to thank the anonymous reviewers for their suggestions that helped improve the paper.
This work was supported in part by the National Key R\&D Program of China (Grant No. 2025YFE0220300), the Beijing Natural Science Foundation Project No. IS26039, CAS Project for Young Scientists in Basic Research (Grant No. YSBR-040), 
and the Engineering and Physical Sciences Research Council  (EPSRC) through grant EP/X042596/1.
\newline\protect\includegraphics[height=8pt]{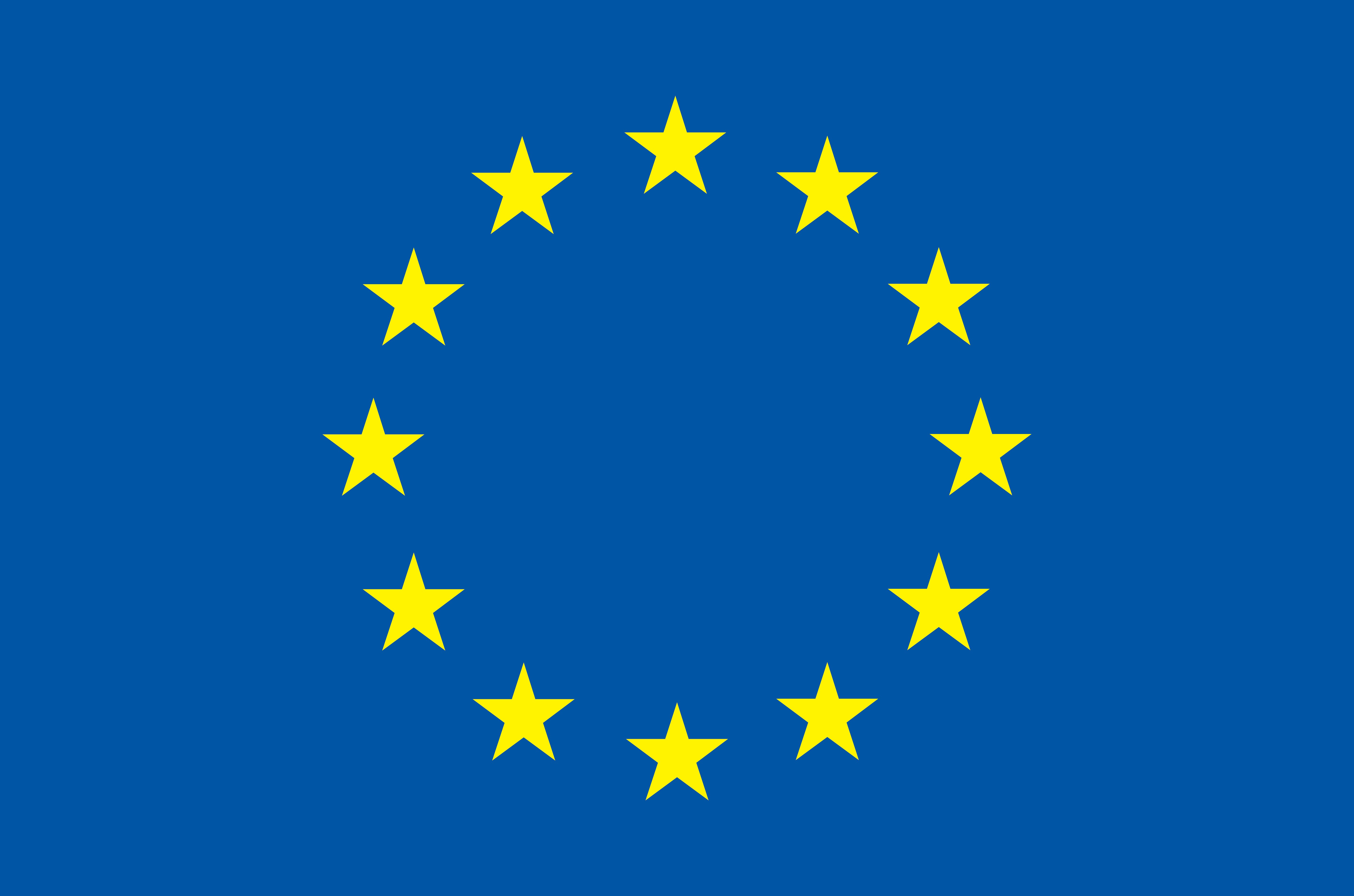} This project is part of the European Union’s Horizon Europe programme under the Marie Skłodowska-Curie grant agreement No. 101208673.

\newpage

\bibliographystyle{splncs04}
\bibliography{ref}

\newpage
\appendix
\section{Missing Proofs of \Cref{sec:non-dense}}\label{app:non-dense}
\subsection{Proof of \Cref{thm:reach}}\label{App:thm_reach_proof}

\thmReachNonDense*
\begin{proof}
Let $\mathcal{A} = (Q, q_0, F, \Delta)$ be a well-typed \dra over $(\mathbb{Z},<)$ with $\kappa$ registers, $(p,v)$ be a configuration and $q$ a state, and let $\min(v)$ and $\max(v)$ denote the smallest and largest integers in $v$, respectively. Suppose $\pi = (p,v) \xrightarrow{u} (q,w)$
for some words $u,w$, where the step-by-step transitions of $\pi$ are:
\[
(p_0,v_0) \xrightarrow{a_1} (p_1,v_1) \xrightarrow{a_2} \dots \xrightarrow{a_n} (p_n,v_n).
\]

The statement of the theorem holds if the following claims are true:
\begin{enumerate}[(C1)]
\item We can obtain a run $\pi_{\mathrm{squeeze}}$ from $\pi$, where
\[
\pi_{\mathrm{squeeze}} = (p_0,x_0) \xrightarrow{b_1} (p_1,x_1) \xrightarrow{b_2} \dots \xrightarrow{b_n} (p_n,x_n)
\]
for some $x_0,\dots,x_n$ and $b_1,\dots,b_n$. Moreover, the range of the values occurring in $x_0,\dots,x_n$ is bounded between $\min(v)-3\kappa\cdot n$ and $\max(v)+3\kappa\cdot n$.

\item If $n > (\sum_{i=0}^\kappa (\max(v)-\min(v)+2)^{i})\cdot |Q| \equiv \textit{len}(\mathcal{A},v)$, then there exists a run $\pi_{\mathrm{short}}$ from $(p,v)$ to $(q,w')$ for some $w'$ with $|\pi_{\mathrm{short}}| < |\pi|$.
\end{enumerate}
With (C1) and (C2), to verify whether there exists a run from $(p,v)$ to $(q,w)$ for some $w$, it suffices to examine runs of length at most $\textit{len}(\mathcal{A},v)$ and over configurations in $Q \times X$, where $X$ consists of words of length between 0 and $\kappa$ and over symbols in $\{\min(v)-\textit{len}(\mathcal{A},v),\dots,\max(v)+\textit{len}(\mathcal{A},v)\} \subset \mathbb{Z}$. Since there are only finitely many such runs, the verification process is decidable.

To derive (C1) and (C2), we show that one can effectively obtain an order-preserving mapping $\sigma$ such that $\sigma(v) = v$ and 
\[
\pi' = (p,\sigma(v)) \xrightarrow{\sigma(u)} (q,\sigma(w))
\]
is a valid run satisfying the following two properties:
(P1) If a value between $\min(v) - 3\kappa \cdot n$ and $\max(v) + 3\kappa \cdot n$ is missing in $u$, then $\max(\sigma(u)) - \min(\sigma(u)) \;<\; \max(u) - \min(u)$.
(P2) If $n > \textit{len}(\mathcal{A}, v)$, then a subrun of $\pi'$ can be removed while the remaining part is still a valid run, a property analogous to the pumping lemma for DFAs.

Consider (C1). Suppose there exist three values $a$, $a'$, and $a''$, satisfying $\max(v) < a < a+1=a' < a'+1=a'' \leq \max(v) + 3\kappa \cdot n$, such that none of $a$, $a'$, or $a''$ appear in any $v_i$ for $i \geq 1$. Let $\sigma^{\textit{shift}}_a$ be the order-preserving partial mapping that maps each value $b$ to itself if $b < a$, and to $b-d$ if $b > a''$, where $d$ is the difference between $a''$ and the minimum value greater than $a''$ among all components in $v_1,\dots,v_n$ and $a_1,\dots,a_n$. To ensure that $\sigma_a^{\textit{shift}}$ is well-defined, we also require $\sigma_a^{\textit{shift}}(b) = a'$ for $b \in \{a, a', a''\}$. Then $\sigma^{\textit{shift}}_a(v) = v$, and
\[
(p_0,\sigma^{\textit{shift}}_a(v_0)) \xrightarrow{\sigma^{\textit{shift}}_a(a_1)} \dots \xrightarrow{\sigma^{\textit{shift}}_a(a_n)} (p_n,\sigma^{\textit{shift}}_a(v_n))
\]
is still a valid run from $(p,v)$ to $(q,w')$ for some $w'$. Moreover, the range of the values occurring in $\sigma^{\textit{shift}}_a(v_1),\dots,\sigma^{\textit{shift}}_a(v_n)$ is strictly smaller than the range of $v_1,\dots,v_n$.
This technique applies symmetrically when $\min(v) - 3\kappa \cdot n \leq a < \min(v)$. By repeatedly applying this process at most $\kappa\cdot n$ times, we obtain the run $\pi_{\mathrm{squeeze}}$ satisfying (C1).

Next, consider (C2). By the pigeonhole principle, if $|\pi| > |Q|$, then along the run there must be a repeated state. Hence, there exists a decomposition $u = u_1 u_2 u_3$ and a state $r$ such that
\begin{equation}\label{app:eq:depump}
\pi = (p,v) \xrightarrow{u_1} (r,x) \xrightarrow{u_2} (r,y) \xrightarrow{u_3} (q,w)
\end{equation}
for some words $x,y$. We could immediately delete the subrun $(r,x) \xrightarrow{u_2} (r,y)$ if $x = y$, and the remaining run $(p,v) \xrightarrow{u_1} (r,x) \xrightarrow{u_3} (q,w)$ would still be valid. However, in general we do not have $x = y$.

More generally, for each $\ell \in \mathbb{N}$, if $|\pi| > \ell \cdot |Q|$, we obtain a decomposition
\[
\pi = (p,v) \xrightarrow{u_1} (r,x_1) \xrightarrow{u_2} (r,x_2) \dots 
\xrightarrow{u_\ell} (r,x_\ell) \xrightarrow{u_{\ell+1}} (q,w)
\]
for some $x_1,\dots,x_\ell$ and $u_1,\dots,u_{\ell+1}$.
For each $i$, let $y_i$ be the sorted version of $x_i$. Then $y_1,\dots,y_\ell$ are strictly increasing sequences of the same length~$\theta$ for some~$\theta$.
Let $y_i = a^{[i]}_1 \dots a^{[i]}_\theta$ for each $i = 1,\dots,\ell$.
If there exist two distinct indices $i<j$ such that, for every $s = 1,\dots,\theta-1$,  
\begin{equation}\label{eq:not_less}
a^{[i]}_{s+1}-a^{[i]}_s \geq a^{[j]}_{s+1}-a^{[j]}_s,
\end{equation}
then by letting $\sigma$ be the mapping $\sigma(a) = $ 
\begin{enumerate}[$\bullet$]
    \item $a + (a^{[i]}_1 - a^{[j]}_1)$ if $a< a^{[j]}_1$.
    \item $a + (a^{[i]}_s - a^{[j]}_s)$ if $a^{[j]}_s \leq a \leq a^{[j]}_{s+1}$ for $s=1,\dots, \theta-1$.
    \item $a+ (a^{[i]}_\theta - a^{[j]}_\theta)$ if $a \geq a^{[j]}_\theta$.
\end{enumerate}
Hence, we have that $\sigma$ is order-preserving, and
\[
(r,\sigma(x_j))
\xrightarrow{\sigma(u_j)} (r,\sigma(x_{j+1})) 
\xrightarrow{\sigma(u_{j+1})} \dots (r,\sigma(x_\ell)) 
\xrightarrow{\sigma(u_{\ell+1})} (q,\sigma(w))
\]
is a valid run from $(r,\sigma(x_j))$ to $(q,\sigma(w))$ on $\sigma(u_j\cdots u_{\ell+1})$, where $\sigma(x_j)=x_i$.
Accordingly,
\[
(p,v) \xrightarrow{u_1} \dots (r,x_i)
\xrightarrow{\sigma(u_j)} (r,\sigma(x_{j+1})) 
\xrightarrow{\sigma(u_{j+1})} \dots (r,\sigma(x_\ell)) 
\xrightarrow{\sigma(u_{\ell+1})} (q,\sigma(w))
\]
is a shorter valid run from $(p,v)$ to $(q,\sigma(w))$.

The following property of finite sequences will be used later:
\begin{lemma}\label{lem:lower_bound}
Let $\delta_1,\dots, \delta_\ell$ be $\ell$ pairwise distinct strictly increasing sequences over $\mathbb{N}$, where each $\delta_i = d^{[i]}_1 \dots d^{[i]}_\theta$ for some fixed $\theta$.
Given a constant $d \in \mathbb{N}$, if $\ell > (d+2)^\theta$, then there exist two distinct indices $i,j$ such that for every $s$, both of values $d^{[i]}_s$ and $d^{[j]}_s$ are greater than $d$ whenever they are not equal.
\end{lemma}
\begin{proof}
    Consider the sequences over $\{0,1,\dots,d,d+1\}$ of length $\theta$. There are exactly $(d+2)^\theta$ such sequences. Among those sequences, there is at least one over exactly $\{0,d+1\}$.
    Given two distinct $\delta_i$ and $\delta_j$, we let $e_{i,j} = f_1\dots f_\theta$ a sequence over $\{0,1,\dots,d,d+1\}$, where $f_s = 0$ if $d^{[i]}_s=\delta^{[j]}_s$; $= t$ if $d^{[i]}_s\neq\delta^{[j]}_s$ and $\min(d^{[i]}_s,\delta^{[j]}_s) = t$; $= d+1$, otherwise. If $\ell > (d+2)^\theta$, there exists $e_{i,j}$ over exactly $\{0,d+1\}$.
    Accordingly, if $\ell > (d+2)^\theta$ there are at least two $\delta_i$ and $\delta_j$ such that $\delta_i\neq \delta_j$ and for each $s$, both of values $d^{[i]}_s$ and $d^{[j]}_s$ are greater than $d$ whenever they are not equal.
\qed
\end{proof}

Now assume that for every distinct pair of $i$ and $j$, there exists $s$ such that  
\begin{equation}\label{eq:ineq_index}
a^{[i]}_{s+1} - a^{[i]}_s \;\neq\; a^{[j]}_{s+1} - a^{[j]}_s.
\end{equation}
By \Cref{lem:lower_bound}, if $\ell > (\max(v)-\min(v)+2)^\theta$, then there exist indices $i,j$ such that both differences $a^{[i]}_{s+1}-a^{[i]}_s$ and $a^{[j]}_{s+1}-a^{[j]}_s$ exceed $\max(v)-\min(v)$ whenever, for every distinct pair $i,j$, there exists an index $s$ satisfying Eq.~\ref{eq:ineq_index}.

We define $\iota$ as an index between $1$ and $\theta$, where:
\begin{enumerate}[$\bullet$]
    \item $\iota = 1$ if $a^{[i]}_1 > \max(v)$,
    \item $\iota = \theta$ if $a^{[i]}_\theta < \min(v)$,
    \item otherwise, $\iota$ is the smallest index satisfying $\min(v) \leq a^{[i]}_\iota \leq \max(v)$.
\end{enumerate}

Let $d_1 \dots d_{\theta-1}$ be the sequence defined by  
\[
d_s = \max\{a^{[i]}_{s+1}-a^{[i]}_s, a^{[j]}_{s+1}-a^{[j]}_s\}
\]
for all $s=1,\dots,\theta-1$.
We define $\sigma$ as follows:
\begin{enumerate}[$\bullet$]
    \item $\sigma(a) = a + a^{[i]}_\iota + (d_\iota + \dots + d_{s-1})$  
          if $a^{[i]}_s \leq a \leq a^{[i]}_{s+1}$ and $s \geq \iota$,
    \item $\sigma(a) = a + a^{[i]}_\iota - (d_{\iota-1} + \dots + d_{s+1})$  
          if $a^{[i]}_s \leq a \leq a^{[i]}_{s+1}$ and $s < \iota$.
\end{enumerate}
Then the following properties hold:
\begin{itemize}
    \item $\sigma$ is order-preserving,
    \item $\sigma(a^{[i]}_{s+1}) - \sigma(a^{[i]}_s) = d_s$ for all $s$,
    \item $\sigma(v) = v$, since there is at most one index $s$ satisfying
          Eq.~\ref{eq:ineq_index} and for that index we have 
          $\min(v) \leq a^{[i]}_s \leq \max(v)$.
\end{itemize}
Accordingly,
\[
\pi_{\mathrm{pre}}
    = (p,\sigma(v))
      \xrightarrow{\sigma(u_1)}
      \dots
      \xrightarrow{\sigma(u_i)}
      (r,\sigma(x_i))
\]
is a valid run from $(p,v)$ to $(r,\sigma(x_i))$ on $\sigma(u_1)\cdots\sigma(u_i)$.
Hence, $\pi_{\mathrm{pre}}$ together with the run $(r,x_j) \xrightarrow{u_{j+1}\cdots u_{\ell+1}} (q,w)$ satisfies Eq.~\ref{eq:not_less}.
\qed
\end{proof}

\subsection{Minimization Algorithm for \wdra{s}}\label{app:min_alg}

\begin{algorithm}[tb]
\caption{Minimization algorithm for \wdra{s}}
\label{alg:min}
\SetKwInOut{Input}{input}
\Input{A \wdra $\mathcal{A}=(Q,q_0,F,\Delta)$ recognizing $L$ over $(\Sigma,R)$}

\For{$q \in Q$}{
	Obtain a word $w_q$ such that $(q_0,\epsilon)\xrightarrow{w_q}_{\mathcal{A}} (q,w)$ for some $w$\;
}
$flag \gets \top$\;
\While{$flag$}{
	$flag \gets \bot$\;
	\For{$p \in Q$}{
		\For{$q \in Q\setminus\{p\}$}{
			\If{$w_p \cong_L w_q$}{
				Merge $p$ and $q$.\;
				$flag \gets \top$\;
			}
		}
	}
}
\For{$q \in Q$}{
	Modify transitions starting from $q$.\;
}
return $\mathcal{A}$\;
\end{algorithm}

By \Cref{thm:myhill}, the minimization problem for \wdra{s} is decidable, provided that the canonical automaton of every given \wdra can be effectively constructed. For DFAs, the canonical automaton can be obtained by merging states that correspond to the same residual language. The same idea applies to \dra{s}, as guaranteed by the following lemma:

\begin{lemma}[\cite{Ley:10:TR}]\label{lem:refinement}
Let $\mathcal{A}$ be a \wdra recognizing $L$. For any two words $u,v$, if
$(q_0,\epsilon) \xrightarrow{u} (q,x)$ and 
$(q_0,\epsilon) \xrightarrow{v} (q,y)$
for some configurations $(q,x)$ and $(q,y)$, then $u \cong_L v$.
\end{lemma}

The pseudocode for \wdra minimization is given in Algorithm~\ref{alg:min}. In the algorithm, we first compute a representative word $w_q$ for each state $q$. Then, for any pair of states $p,q$, we merge them whenever $w_p \cong_L w_q$. After this step, the automaton is already state-minimal. However, it may still fail to be data-minimal. Therefore, in a final step, we use $mem_L(w_q)$ for each state $q$ to adjust transitions outgoing from $q$ so as to achieve data-minimality.

The decidability of the following three problems ensures that Algorithm~\ref{alg:min} is an effective minimization procedure:
\begin{enumerate}[$\bullet$]
\item \textbf{Representative problem:} Given a state $q$, compute a representative word $w_q$ such that $(q_0,\epsilon)\xrightarrow{w_q} (q,w)$ for some word $w$, whenever such a representative exists.
\item \textbf{Memorability problem:} Given a word $u$, compute $mem_L(u)$.
\item \textbf{Word-equivalence problem:} Given two words $u,v$, verify whether $u \cong_L v$.
\end{enumerate}

Clearly, the representative problem is decidable if the configuration reachability problem is decidable and a witness word $w_q$ can be effectively extracted.

As for DFAs, the equivalence problem of two \dra{s} reduces to the reachability problem for their product automaton. 
\begin{definition}\label{def:product-automaton}
Given two \dra $\mathcal{A}_1 $ and $ \mathcal{A}_2$, the product automaton $\mathcal{A}_1 \times \mathcal{A}_2$ is a tuple $\big(Q^{\times}, q_0^{\times}, F^{\times}, \Delta^{\times}\big)$ 
where:
\begin{itemize}
\item $Q^\times = Q^1 \times Q^2$ is a set of states.
\item $(q_0^1, q_0^2) \in Q^{\times}$ is the \emph{initial state}.
\item $F^\times = F^1 \times F^2$ is the set of \emph{final states}.
\item $\Delta^\times$ is a finite set of \emph{transitions} of the form $\big( (p_1, p_2), \allowbreak (\tau_1, \tau_2), \allowbreak (E_1,E_2), \allowbreak (q_1, q_2) \big)$, 
where 
$(p_i, \tau_i, E_i, q_i) \in \Delta^i$ for $i \in \{1,2\}$.
\end{itemize}
\end{definition}

This product register automaton differs from the RAs in \Cref{def:RA}: it is not well-typed, since its registers may contain duplicate values. However, it can be translated into a \dra satisfying \Cref{def:RA} by introducing additional states that simulate the register configuration so as to prevent duplicate values.
Thus, the equivalence problem for \dra{s} from given configurations is decidable whenever the configuration reachability problem is decidable.

Given a word $u$, if $(q_0,\epsilon)\xrightarrow{u}_{\mathcal{A}}(q,w)$, then $mem_L(u)$ is a subsequence of $w$, obtained by deleting all non-memorable symbols. If every extension run $(q,w)\xrightarrow{s}(r,x)$ from $(q,w)$ satisfies $r\notin F$, then $mem_L(u)=\epsilon$. Otherwise, suppose $(q,w)\xrightarrow{s}(r,x)$ for some $r\in F$, and let $w'=\sigma(w)$ for some order preserving mapping $\sigma$ which maps $a$ to a different value $b$. If the \dra started from $(q,w)$ is not equivalent to that started from $(q,w')$, then $a$ is memorable; otherwise, it is not.
When $\Sigma$ is dense, such a $w'$ always exists. This need not hold for non-dense alphabets. For instance, if $w=2\cdot 3\cdot 4$ over $\mathbb{Z}$ and $a=3$, then there is no $b\neq 3$ such that $\sigma$ preserves the ordering relations in $w$. In contrast, if $w=1\cdot 3\cdot 5$ and $a=3$, we may choose $b=4$.

Recall that $w$ arises from a run $(q_0,\epsilon)\xrightarrow{u}(q,w)$, and if $u\cong_L u'$ with $(q_0,\epsilon)\xrightarrow{u'}(q,w')$, then the $i$th symbol of $w$ is memorable in $u$ iff the $i$th symbol of $w'$ is memorable in $u'$. Since $L$ is a data language, for every $u\sim_R u'$, we have $u\cong_L u'$. For any sequence $u$, we can effectively compute a sequence $u'$ such that $u'\sim_R u$ and any two distinct symbols in $u'$ differ by more than $1$. For example, if $u=1\cdot 2\cdot 3$, we may take $u'=1\cdot 3\cdot 5$. Hence the memorability problem reduces to the configuration equivalence problem, and thereby to the configuration reachability problem.

Suppose there exists an $R$-preserving mapping $\sigma$ such that 
$\sigma(mem_L(u)) = mem_L(v)$ for given words $u$ and $v$, and let $u' = \sigma(u)$. Then $u\cong_L v$ iff for all words $s$, $u'\cdot s\in L \Longleftrightarrow v\cdot s\in L$. Let $(q_{u'},w_{u'})$ and $(q_v,w_v)$ be the configurations reached from $(q_0,\epsilon)$ upon reading $u'$ and $v$, respectively. Therefore, $u\cong_L v$ iff the \dra $\mathcal{A}$ starting from $(q_{u'},w_{u'})$ is equivalent to the one starting from $(q_v,w_v)$. Consequently, the word-equivalence problem, and thus the minimization problem for \wdra{s}, is decidable because the configuration reachability problem is decidable.

\section{More Examples on the Learning Algorithm}
\label{app:examples}

\begin{figure}
    \centering
    \includegraphics[]{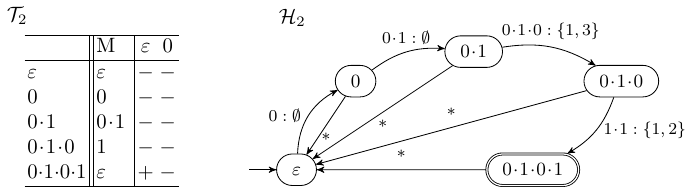}
    \caption{Closed Table $\obtable_2$ and its derived DRA $\hypothesis_2$}
    \label{fig:table-2-and-dra-2}
\end{figure}

\begin{example}
Let us reconsider the examples in Fig.~\ref{fig:running-example-tables-dras}. Assume that the teacher returns counterexample $w = 4\mdot 5\mdot 4\mdot 5$ to the learner on equivalence query $\EQ(\hypothesis_1)$.
Clearly, $w \in L$ but $w \notin \lang(\hypothesis_1)$.
Initially, we have $u = \emptyword$, $r = \emptyword$ and $\textit{mq} = +$.
In the first iteration, we decompose $w = a\cdot z$ with $a = 4$ and $z = 5\mdot 4\mdot 5$.
Reading $a = 4$, the transition $(\emptyword, \emptyword) \xrightarrow{0:\emptyset} (0, 4)$ is taken when reading $a$, so $b = 0$.
By applying $\sigma_1$ which maps $4$ to $0$, we obtain $\sigma_1(z) = 1\mdot 0\mdot 1$.
Furthermore, $\sigma_2$ maps $M(u\cdot b) = 0$ to $M(u') = M(0) = 0$, hence $\sigma$ is an identity map.
Thus, $\sigma_2(\sigma_1(z)) = 1\mdot 0\mdot 1$ and $\MQ(u'\cdot \sigma_2(\sigma_1(z))) = \MQ(0\mdot 1\mdot 0\mdot 1) = +$.
We then set $u = 0, r = 0$ and $w = 1\mdot 0\mdot 1$ and go to next iteration.
In the second iteration, we again write $w = a\cdot z$ with $a = 1$ and $z = 0\mdot 1$.
The transition $(0, 0) \xrightarrow{0\cdot 1: \emptyset} (0\mdot 1, 0\mdot 1)$ is taken.
Since $M(u)\cdot a = M(u)\cdot b = M(u')$, $\sigma_1$ and $\sigma_2$ are both identity mappings.
Thus, $\sigma_2(\sigma_1(z)) = 0\mdot 1$ and $\MQ(u' \mdot \sigma_2(\sigma_1(0\mdot 1))) = \MQ(0\mdot 1 \mdot 0\mdot 1) = +$.
We set $u = 0 \mdot 1, r = 0\mdot 1$ and $w = 0\mdot 1$ and go to the next iteration.
In this last iteration, the same reasoning applies: $a = 0$ and $z = 1$ and the transition $(0\mdot 1, 0\mdot 1) \xrightarrow{0\mdot 1\mdot 0: \{1,3\}} (0, 1)$ is taken; so $b = 0 $.
Since $a = b$ and thus $M(u)\cdot a = M(u)\cdot b$, $\sigma_1$ is again the identity mapping.
We have $M(u\cdot b) = M(0\mdot 1\mdot 0) = 1$ and $M(u') = M(0) = 0$, giving $v = \sigma_2(\sigma_1(z)) = 0$.
Thus, $\MQ(u'\cdot v) = \MQ(0\mdot 0) = -$.
Therefore the returned word is $v = 0$ which will be added into $S$.
After closing the updated table, we will obtain the table $\obtable_2$ and derive a new conjecture DRA $\hypothesis_2$, as shown in Fig.~\ref{fig:table-2-and-dra-2}.
After refining $\hypothesis_2$ with counterexample $1\mdot 0 \mdot 1\mdot 2 \mdot 1\mdot 2$, we will get the final table $\obtable$ and construct from $\obtable$ the correct conjecture DRA $\target$ in Fig.~\ref{fig:table-and-dra}. \qed
\end{example}

\section{More Details on Learning DRAs over Non-dense Domains}
\label{app:non-dense-learning}

We mentioned in the main content that one can learn DRAs over $(\alphabetZ, <)$ with our DRA learner for $(\alphabetQ, <)$. Here we provide more details on that.

Assume that the teacher answers only membership and memorability queries on words over $\alphabetZ$, and equivalence query over a DRA over $\alphabetZ$.
Let $L_{\alphabetZ}$ be our target language over $\alphabetZ$ and $L_{\alphabetQ}$ be its corresponding language over $\alphabetQ$.
We can just add an interface between our DRA learner of dense domains and the teacher for non-dense domains.
\begin{itemize}
    \item For asking a membership query for a word $u$ over $\alphabetQ$, we can use a bijective order-preserving mapping $\sigma$ that maps a word over $\alphabetQ$ to a word over $\alphabetZ$.
    For instance, when $u = 1\cdot 2\cdot \frac{1}{2}$, we have $\sigma(u) = 2\cdot 4\cdot 1$.
    We then return the response on $\MQ(\sigma(u))$.
    Since $u \sim_R \sigma(u)$, they should have the same membership in the target language $L_{\alphabetQ}$.

    \item For asking a memorability query for a word $u$ over $\alphabetQ$, we also ask the memorability query $\Mem(\sigma(u))$ instead and return $\sigma^{-1}(\Mem(\sigma(u)))$.
    Since $u \sim_R \sigma(u)$ over the dense domains, we know that the memorable words for $\sigma(u)$ have the same positions as the memorable words for $u$.

    \item To resolve the equivalence query on the conjectured DRA $\hypothesis$ over $\alphabetQ$, we convert it to a DRA $\hypothesis'$ over $\alphabetQ$ for all word types of transitions using $\sigma$ and ask equivalence qeury $\EQ(\hypothesis')$.
    If the answer is ``Yes", we output the DRA $\hypothesis'$ over $\alphabetZ$, otherwise, just return the counterexample $w$ to our DRA learner, since it is also a valid counterexample over $\alphabetQ$.
\end{itemize}

By Theorem~\ref{thm:Z_eq_Q}, this learning algorithm will output a correct DRA of the target language $L_{\alphabetZ}$.

\section{Missing Proofs of \Cref{sec:learning-mem}}
\label{app:learning}

\witnessNotEq*

\begin{proof}
We first prove the ``only if'' direction. 
Suppose that $u_1 \not\cong_L u_2$. Since $mem_L(u_1) \sim_R mem_L(u_2)$, by Definition~\ref{def:eq-relation-data-languages}, there exist words $x, y$ such that $mem_L(u_1) \cdot x \sim_R mem_L(u_2) \cdot y$, but exactly one of the words $u_1 \cdot x$ and $u_2 \cdot y$ is in $L$. Because $\Sigma$ is dense, the relation $mem_L(u_1) \cdot x \sim_R mem_L(u_2) \cdot y$ implies the existence of an order-preserving mapping $\sigma$ such that $\sigma(mem_L(u_1) \cdot x) = mem_L(u_2) \cdot y$. By letting $w = \sigma(x) = y$, we obtain that either $\sigma(u_1) \cdot w \in L$ or $u_2 \cdot w \in L$.

Next, we prove the ``if'' direction. 
Suppose there exist a word $w$ and an order-preserving mapping $\sigma$ such that $\sigma(mem_L(u_1)) = mem_L(u_2)$ and $\sigma(u_1) \cdot w \in L \not\Leftrightarrow u_2 \cdot w \in L$. 
Let $x = \sigma^{-1}(w)$ and $y = w$. Since $\sigma$ is order-preserving and $\Sigma$ is dense, we have $mem_L(u_1) \cdot x \sim_R mem_L(u_2) \cdot y$. Moreover, exactly one of the words $u_1 \cdot x$ and $u_2 \cdot y$ is in $L$. Hence, $u_1 \not\cong_L u_2$.
\qed
\end{proof}

To order to prove Lemma~\ref{lem:inequivalence-S}, it suffices to prove the following Lemma~\ref{lem:one_mapping}, since Lemma~\ref{lem:inequivalence-S} is a direct consequence of Lemma~\ref{lem:one_mapping}.

\begin{restatable}{lemma}{oneMapping}\label{lem:one_mapping}
    Let $L$ be a data language over $(\Sigma, R)$ where $\alphabet$ is dense. For a word $u \in \Sigma^*$ and two bijective order-preserving mappings $\sigma, \sigma'$ with $\sigma(mem_L(u)) = \sigma'(mem_L(u))$, it holds that $\sigma(u) \cdot w \in L \Leftrightarrow \sigma'(u) \cdot w \in L$ for all $w \in \Sigma^*$.
\end{restatable}

\begin{proof}
Since $\sigma$ and $\sigma'$ are bijective, their inverses $\sigma^{-1}$ and $\sigma'^{-1}$ exist and are also order-preserving.
Let $u$ be a word.
By the assumption $\sigma(mem_L(u)) = \sigma'(mem_L(u))$, it follows that for every word $w \in \finwords$,  $mem_L(u) \cdot \sigma^{-1}(w) \sim_R \sigma(mem_L(u) \cdot \sigma^{-1}(w)) = \sigma(mem_L(u)) \cdot w = \sigma'(mem_L(u)) \cdot w \sim_R \sigma'^{-1}(\sigma'(mem_L(u)) \cdot w) = mem_L(u) \cdot \sigma'^{-1}(w)$ holds.
That is, $mem_L(u)\cdot \sigma^{-1}(w) \sim_R mem_L(u) \cdot \sigma'^{-1}(w)$ holds for every $w \in \finwords$.
By definition of $\cong_L$, since $mem_L(u)\cdot \sigma^{-1}(w) \sim_R mem_L(u) \cdot \sigma'^{-1}(w)$, we have $u \cdot \sigma^{-1}(w) \in L \Leftrightarrow u \cdot \sigma'^{-1}(w) \in L$ for all $w \in \finwords$ because $ u \cong_L u$ holds.
It then follows that, $\sigma(u) \cdot w \in L \Leftrightarrow \sigma'(u) \cdot w \in L$ for all $w \in \finwords$ since $\sigma$ and $\sigma'$ are order-preserving mappings.
\qed
\end{proof}

\inequivalenceS*
\begin{proof}
    It is a direct consequence of Lemma~\ref{lem:one_mapping}.
    \qed
\end{proof}

The following lemma, which captures the key properties of \scn, follows directly from its definition.
\begin{lemma}\label{lem:S-equivalence-props}
Let $S \subseteq \Sigma^{*}$ be a set of words.
We have:
\begin{enumerate}
    \item if $u$ is not \scap $v$, then $u \not\cong_L v$;
    \item if $u$ is not \scap $v$, then $u$ is not \scap $v$ whenever $S \subseteq S'$.
\end{enumerate}
\end{lemma}

We present examples below to show that \scn is \emph{not} an equivalence relation: it is neither symmetric nor transitive.
\begin{example}\label{example:Sconsistency-not-equivalence}
    Consider the language $L_4$ where $n = 4$ for the language in \Cref{ex:inc_dec_sequences}, that is, $L_4 = \{w \in \alphabetQ \mid |w| = 4 \text{ and } w \text{ is strictly increasing or decreasing}\}$.

    \begin{center}
    \begin{tabular}{l||l|c|c}
     & $M$ & $\varepsilon$ & $1 \cdot 2$ \\
    \hline
    $\varepsilon$ & $\varepsilon$ & $-$ & $-$ \\
    $0$ & $0$ & $-$ & $-$ \\
    $0 \cdot 1$ & $1$ & $-$ & $-$\\
    \hline\hline
    \textcolor{gray}{$0 \cdot 1 \cdot 2$} & \textcolor{gray}{$2$} & \textcolor{gray}{$-$} & \textcolor{gray}{$-$} \\
    \end{tabular}
    \end{center}
    
    During the learning process, the above observation table may arise. 
    The prefixes of the first three rows belong to $U$, whereas the last row, $0 \cdot 1 \cdot 2$, belongs to $X$, 
    the only one-letter extension in $X$ shown in this table. 
    We denote by $u_i$ the prefix corresponding to the $i$th row.
    This table witnesses \emph{non-symmetry}: $u_2$ is \scap $u_3$, but $u_3$ is not \scap $u_2$.
    In addition, the one-letter extension $0 \cdot 1 \cdot 2$ (in $X$) is \scap both $u_2$ and $u_3$,
    witnessing \emph{non-transitivity} of \scn.
\end{example}


In order to prove Lemma~\ref{lem:correct-ness-cex-analysis}, we first prove the following Lemma.
\begin{lemma}\label{lem:same-run}
Let $w = a_1\cdots a_d$ and $\pi = (u_0, r_0) \xrightarrow{a_1} \cdots \xrightarrow{a_d} (u_d, r_d)$ be the run of $\hypothesis$ on $w$.
Let $u'_0, \cdots, u'_d$ be the sequence of states visited in Algorithm~\ref{algo:cex-analysis} if we remove the \textbf{if} statement (i.e., we do not use the early return).
It then holds that $u'_i = u_i$ for all $0 \leq i \leq d$.
\end{lemma}
\begin{proof}
    Let $r'_0, \cdots, r'_d$ and $a'_1 z_0,\cdots, a'_d z_{d-1}$ be the corresponding register values and the remaining suffixes visited in Algorithm~\ref{algo:cex-analysis} where $z_d = \emptyword$ and $a'_1z_0 = w$.
    Let $r'_0 b_1, \cdots, r'_d b_d$ be the selected word type visited in Algorithm~\ref{algo:cex-analysis}.
    We first prove the following invariants in the whole loop by induction:
    \begin{enumerate}
        \item[(i)] $r'_i \cdot a'_{i+1} \sim_R M(u'_i)\cdot b_{i+1}$ for all $0 \leq i < d$.
        \item[(ii)] $r'_i a'_{i+1} z_i \sim_R r'_i b_{i+1} \sigma_1(z_i)$ for all $0 \leq i < d$.
    \end{enumerate}
For the base case, obviously, when $i = 0$, we have $u'_0 = r'_0$ and $z_0 = w[2\cdots]$.
It follows that $M(u'_0) = \emptyword$ and thus $r'_0\cdot a'_{1} \sim_R M(u'_0)\cdot b_1$ holds directly.
Moreover, $\sigma_1$ is in fact an order-preserving mapping from $a'_{1} = a_1$ to $b_1$.
Hence, $\sigma_1(a'_{1}z_0) = b_1\cdot \sigma_1(z_0) \sim_R a'_{1}z_0$.

For the induction step, assume that it holds for all $i \le \ell$.
According to Algorithm~\ref{algo:cex-analysis}, $r'_{\ell+1} = M(u'_{\ell+1}) $ and $z_{\ell + 1} = \sigma_2(\sigma_1(z_\ell))$.
Since when reading $a'_{\ell+2}$ at configuration $(u'_{\ell+1}, r'_{\ell+1}=M(u'_{\ell+1}))$, the transition with the word type $M(u'_{\ell+1})\cdot b_{\ell+2}$ is selected, it indicates that $r'_{\ell+1}\cdot a'_{\ell+2} = M(u'_{\ell+1})\cdot a'_{\ell+2} \sim_R M(u'_{\ell+1})\cdot b_{\ell+2}$.
Recall that current $\sigma_1$ is an order-preserving mapping from $r'_{\ell+1} \cdot a'_{\ell+2}$ to $r'_{\ell+1}\cdot b_{\ell+2}$.
Therefore, it holds that $r'_{\ell+1}\cdot a'_{\ell+2} \cdot z_{\ell+1} \sim_R \sigma_1(r'_{\ell+1}\cdot a'_{\ell+2} \cdot z_{\ell+1}) = r'_{\ell+1}\cdot b_{\ell+2}\cdot \sigma_1(z_{\ell+1})$.

Since $r'_i a'_{i+1} z_i \sim_R r'_i b_{i+1} \sigma_1(z_i)$ for all $0 \leq i < d$, we have that $u'_i a'_{i+1} z_i \in L$ if and only if, $u'_i b_{i+1} \sigma_1(z_i) \in L$.
The reason is the following.
Assume that $\target$ is the minimal well-typed DRA of $L$ defined by $\cong_L$.
After reading $u'_i$, $\target$ would reach a configuration $(q, r)$ where $r = r'_i$.
Since $r'_i a'_{i+1} z_i \sim_R r'_i b_{i+1} \sigma_1(z_i)$, if from $(q, r)$, $a'_{i+1} z_i$ is accepted, $\target$ also accepts $b_{i+1} \sigma_1(z_i)$.

    Next, we prove by induction that $u'_i = u_i$ and $r'_i a'_{i+1} z_i \sim_R r_i w[i+1\cdots]$ for all $0\leq i < d$.

    For the base case, obviously, when $i = 0$, we have $u'_0 = u_0 = \emptyword$.
    Moreover, $a'_{1}z_1  = w$.
    Since $r'_0 = r_0 = M(u'_0)  = \emptyword$, the claim holds.
    
    Now, for the induction step, assume that $u'_\ell = u_\ell$, $r'_\ell\cdot a'_{\ell + 1} z_{\ell} \sim_R r_\ell \cdot w[\ell+1\cdots]$.
    We want to prove that the claim also holds at $\ell + 1$.
    Let $(u_\ell, M(u_\ell)\cdot a_{\ell+1}, E_\ell, u_{\ell+1})$ be the transition taken in $\hypothesis$ when reading $a_{\ell +1} = w[\ell+1]$ at $(u_{\ell}, r_{\ell})$ and $(u'_{\ell}, M(u'_{\ell})\cdot b_{\ell+1}, E'_\ell, u'_{\ell+1})$ be the transition taken in Algorithm~\ref{algo:cex-analysis} when reading $a'_{\ell+1} $ at $(u'_{\ell}, r'_{\ell})$.
    By induction hypothesis, we know that $u_{\ell} = u'_{\ell}$ and $r'_{\ell} a'_{\ell+1} z_{\ell} \sim_R r_{\ell}\cdot w[\ell+1\cdots]$.
    It follows that $r'_{\ell} \cdot a'_{\ell+1} \sim_R r_{\ell}\cdot w[\ell+1] = r_{\ell}\cdot a_{\ell+1}$ (the prefixes also have the same type).
    On the other hand, $M(u_\ell)\cdot a_{\ell+1} = r_{\ell} \cdot a_{\ell +1} \sim_R r'_{\ell}\cdot a'_{\ell+1} \sim_R r'_{\ell}\cdot b_{\ell+1}$ hold.
    So, $M(u_\ell)\cdot a_{\ell+1} \sim_R M(u'_{\ell})\cdot b_{\ell+1}$.
    We then know that the transition $(u_\ell, M(u_\ell)\cdot a_{\ell+1}, E_\ell, u_{\ell+1})$ and the transition $(u'_{\ell}, M(u'_{\ell})\cdot b_{\ell+1}, E'_\ell, u'_{\ell+1})$ are the same transition since $u'_{\ell} = u_{\ell}$ and $\hypothesis$ is deterministic.
    Therefore, $u_{\ell + 1} = u'_{\ell+1}$.
    In fact, by Claim (ii), we have $r'_{\ell} b_{\ell+1} \sigma_1(z_\ell) \sim_R r'_{\ell}\cdot a'_{\ell+1}\cdot z_{\ell}$.
    Since $r'_{\ell}\cdot a'_{\ell+1}z_{\ell} \sim_R r_{\ell} \cdot w[\ell+1\cdots]$, we also have that $r'_{\ell} b_{\ell+1} \sigma_1(z_{\ell}) \sim_R r_{\ell} \cdot w[\ell+1\cdots]$.
    By removing the letters in the same positions from $r'_{\ell}\cdot b_{\ell+1}$ and $r_{\ell} \cdot a_{\ell+1}$, we know that $M(r'_{\ell}\cdot b_{\ell+1})\cdot \sigma_1(z_{\ell}) \sim_R M(r_{\ell} \cdot a_{\ell+1})\cdot w[\ell+2\cdots]$.
    Further, we have $M(r'_{\ell}\cdot b_{\ell+1})\cdot \sigma_1(z_{\ell}) \sim_R \sigma_2(M(r'_{\ell}\cdot b_{\ell+1})\cdot \sigma_1(z_{\ell})) = M(u'_{\ell+1}) \cdot \sigma_2(\sigma_1(z_{\ell})) = r'_{\ell+1} \cdot a'_{\ell+2}  z_{\ell+1}$ and $M(r_{\ell} \cdot a_{\ell+1})\cdot w[\ell+2\cdots] = r_{\ell+1}\cdot w[\ell+2\cdots]$.
    Therefore, it holds that $u'_{i} = u_{i}$ and $r'_{i}a'_{i+1}z_i \sim_R r_{i}\cdot w[i+1]$ for all $0\leq i < d$.
    
    With similar argument, we can also prove that $u'_d = u_d$. Therefore, we have completed the proof.
\qed
\end{proof}

\correctnessCEXanalysis*
\begin{proof}
By \Cref{lem:same-run}, the state sequence over $w$ and that in Algorithm~\ref{algo:cex-analysis} coincide.
Therefore, for the last state $u_d$, we have $\MQ(u_d) \neq \textit{mq}$.
Therefore, the membership results must flip at some point before reaching the last state $u_d$ and the if condition will be satisfied. 
Hence, \Cref{algo:cex-analysis} terminates after at most $|w|$ iterations.

Assume it returns at the $i$th iteration with 
$v = \sigma_{2}(\sigma_{1}(z_{i-1}))$ and some $u_i'$ such that  
\[
  \MQ(u_i' \cdot v) \neq \MQ(u_{i-1}' \cdot a_i' z_{i-1}),
\]
where the one-letter extension $u_{i-1}' b_i$ is currently \scap $u_i'$.
Since
\[
  \sigma_{2}(\sigma_{1}(M(u_{i-1}' a_i'))) 
  = \sigma_{2}(M(u_{i-1}' b_i)) 
  = M(u_i'),
\]
we obtain
\[
  \sigma_{2}(\sigma_{1}(M(u_{i-1}' a_i'))) \cdot \sigma_{2}(\sigma_{1}(z_{i-1}))
  = M(u_i') \cdot v.
\]

Moreover, by proof of \Cref{lem:same-run},  
\[
  u_{i-1}' a_i' z_{i-1} \in L 
  \iff 
  u_{i-1}' b_i \sigma_{1}(z_{i-1}) \in L,
\]
and hence,
\[
  \MQ(u_{i-1}' a_i' z_{i-1})
  = \MQ(u_{i-1}' b_i \sigma_{1}(z_{i-1}))
  = \MQ(\sigma_{2}(u_{i-1}' b_i) \cdot v)
\]

Thus, $ \MQ(\sigma_{2}(u_{i-1}' b_i) \cdot v) \neq \MQ(u_i' \cdot v)$, and $u_{i-1}' b_i$ is not $S \mathbin{\uplus} \{v\}$-consistent with $u_i'$, witnessed by the word $v$ and the mapping $\sigma_{2}$. 
\qed
\end{proof}

We now provide the missing proof for Theorem~\ref{thm:correctness-learning-algorithm}.

\correctnessLearning*

\begin{proof}

When \Cref{algo:cex-analysis} returns a distinguishing suffix $v$ separating $ub$ from the state $u'$ with which $ub$ is currently $S$-consistent, 
we add $v$ to the suffix set $S$ to correct the transition from $u$ over $b$.  
Consequently, $ub$ becomes not $S \uplus \{v\}$-consistent with $u'$.  
Each transition can be corrected at most $n-1$ times; 
that is, the target state of a transition over $b$ from $u$ may be misclassified no more than $n-1$ times, 
as once $ub$ is inconsistent with a state $u''$, it remains inconsistent with $u''$ thereafter by \Cref{lem:S-equivalence-props}.
Each counterexample will correct at least one transition.
Hence, the learner asks at most $m \cdot (n-1)$ equivalence queries in total.
It is worth noting that in our experiments, new states will usually be discovered after adding the suffix returned from the counterexample analysis procedure.

For memorability queries, we only need it for every outgoing transition.
The reason is as follows:
first, for the initial state $\emptyword$, $M(\emptyword)$ is trivially $\emptyword$, so there is no need to ask memorability queries for the empty word.
The only place we need to ask memorability queries is when we obtain a new one-letter extension and add it to the extension set $X$.
(All new states come from $X$, so their memorable words have been queried when first added into $X$.)
Since the number of all one-letter extensions is equivalent to the number of transitions, the number of memorability queries is $m$.
So, the learner needs $m$ memorability queries.

For membership queries, we need (1) $\bigO(n \times n\times m)$ queries to fill the table entries since there are at most $m\times n$ counterexamples/suffixes, 
(2) $\bigO(d \times n \times m)$ membership queries to analyze the counterexamples where each analysis costs at most $d$ membership queries and we need to analyze counterexamples at most $m\times n$ times, and (3) $ \bigO(m^3\times n^3)$ membership queries to look for the \scn class for each extension in $X$ when constructing the conjectured DRA where $|X| \leq m$, each transition will cost at most $n \times n \times m $ times during a DRA construction with $n$ being the maximum row number and $m\times n$ being maximum column number and we need to construct conjectured DRAs at most $m\times n$ times.
Therefore, we need in total $\bigO(d\times m\times n + m^3 \times n^3)$ membership queries.
\qed
\end{proof}

\section{Preliminary Experiments}
\label{app:expr}
We implemented the proposed active learning algorithm in Python\footnote{\url{https://github.com/liyong31/RALearning.git}} and evaluated it on all examples presented in this paper, as well as on randomly generated DRAs. All experiments were conducted on a machine running Ubuntu 24.04 LTS, equipped with an Intel i7-11800H CPU and 16\,GB of RAM.

\noindent\textbf{Query Complexity.}
Due to space constraints, we report results only for the family of DRAs introduced in \Cref{ex:inc_dec_sequences}, namely the languages $L_n$ for $n = 1, \ldots, 25$. The minimal \wdra recognizing $L_n$ has three states and three transitions when $n = 1$, and $2n$ states and $6n - 6$ transitions when $n > 1$. As shown in \Cref{fig:Ln}, the numbers of all query types grow polynomially with the size of the target DRA (measured by $2n$ states and $6n - 6$ transitions). Among them, the number of $\MQ$s grows the fastest, followed by $\EQ$s and then $\Mem$ queries, in line with the bounds in \Cref{thm:correctness-learning-algorithm}. Additional results are reported in \Cref{app:experiments}.

\begin{figure}[t]
    \centering
    \begin{subfigure}{0.49\textwidth}
        \centering
        \includegraphics[width=1\linewidth]{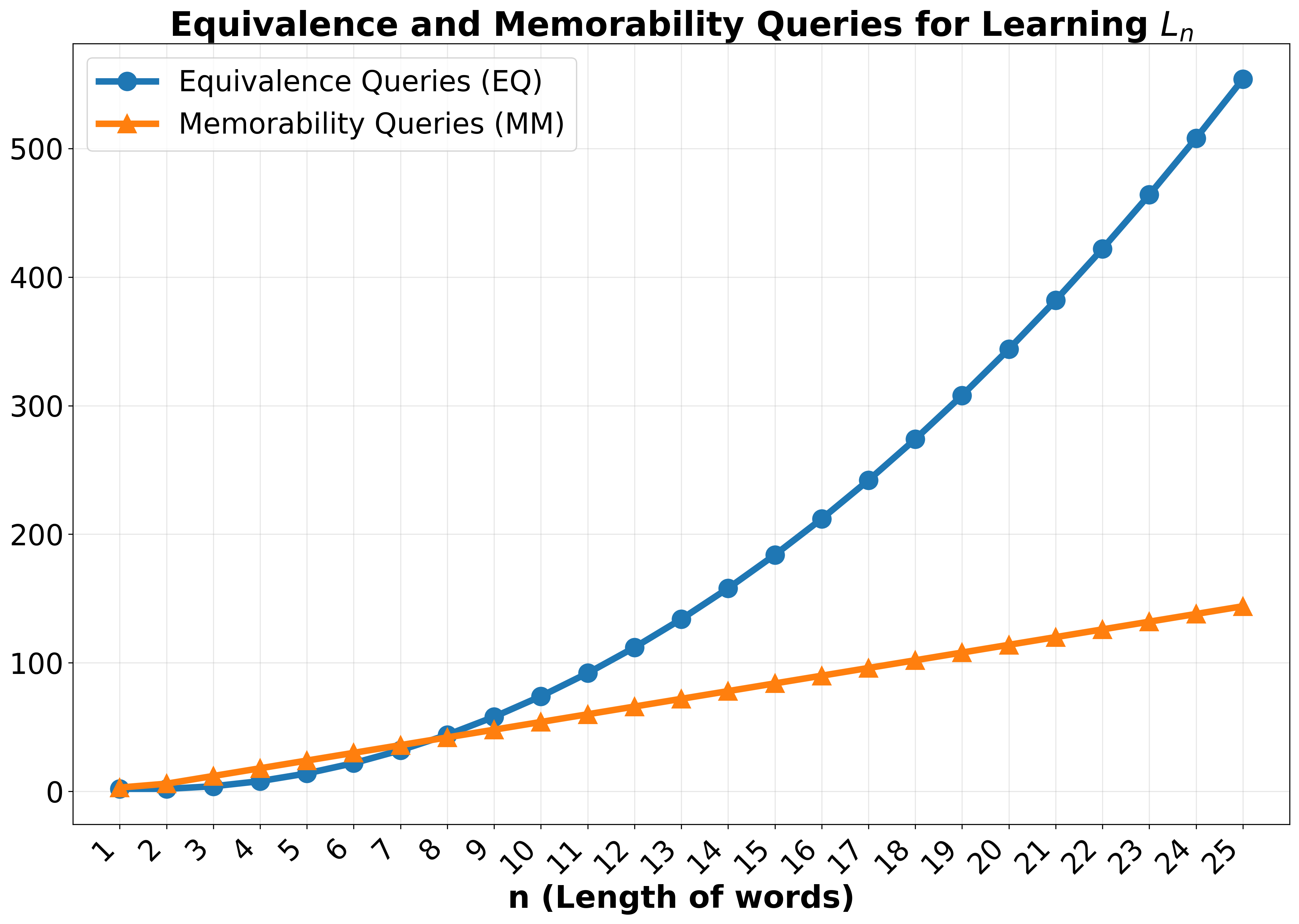}
    \end{subfigure}
    \hfill
    \begin{subfigure}{0.49\textwidth}
        \centering
        \includegraphics[width=1\linewidth]{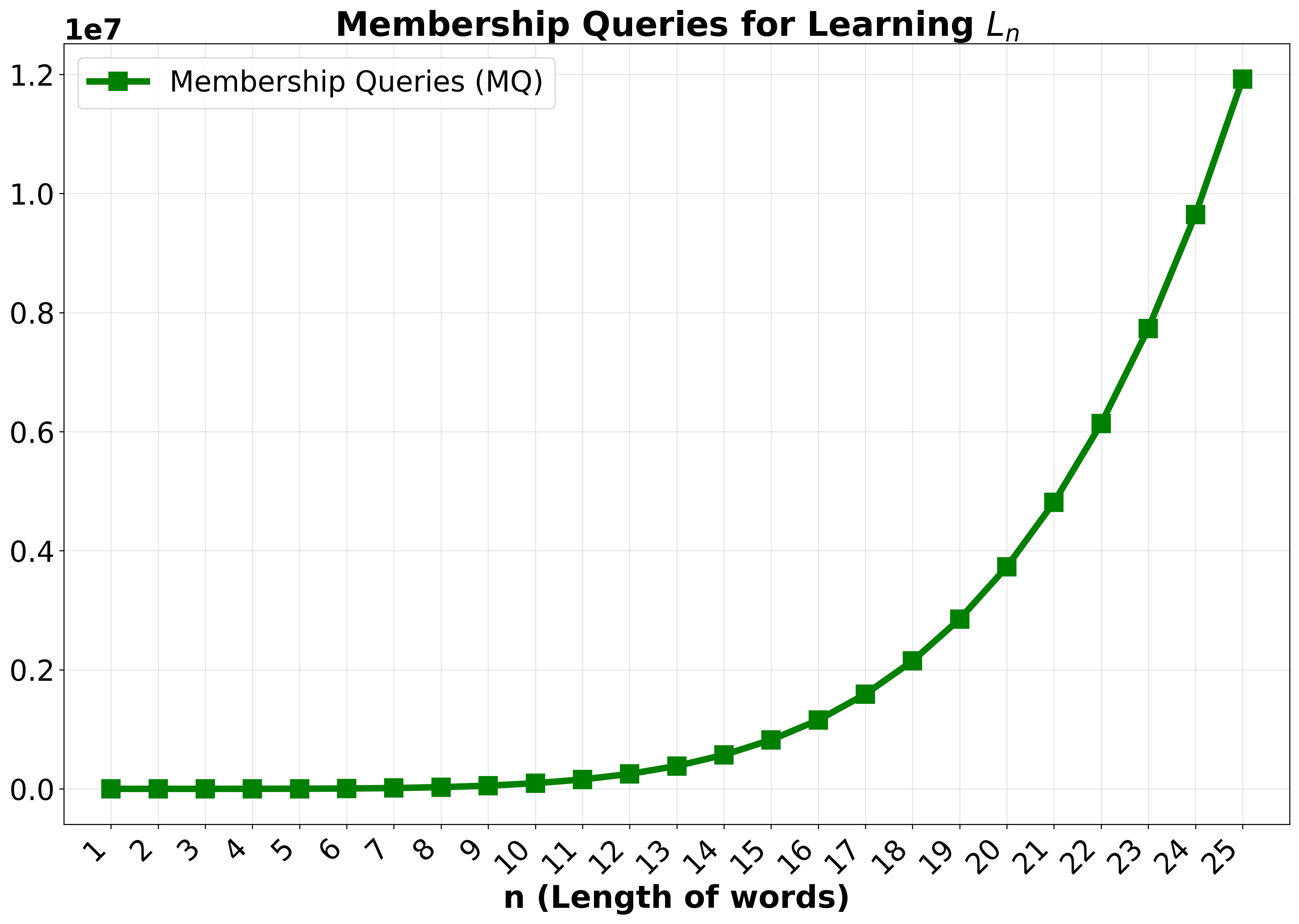}
    \end{subfigure}
    \caption{The x-axis is the the length of word $n = 1, \ldots, 25$ for $L_n$. Left: the number of $\EQ$s and $\Mem$s for each $n$. Right: the number of $\MQ$s.}
    \label{fig:Ln}
\end{figure}

\noindent\textbf{Evaluation of \ralib\ and our approach.}
We also conducted experiments on $L_n$ using \ralib~\cite{DBLP:journals/fac/CasselHJS16, Cassel2015RALibA}, using the \texttt{iosimulator} backend and default settings on $L_n$ for $2 \leq n \leq 10$, without imposing a time limit. However, a direct comparison is not straightforward: \ralib\footnote{\url{https://github.com/LearnLib/ralib}} targets register extended finite state machines (EFSMs)~\cite{DBLP:journals/fac/CasselHJS16}, which incorporate both inputs and outputs, whereas our model does not support outputs. Moreover, EFSMs characterize safety languages and do not explicitly distinguish between accepting and rejecting states.

To enable comparison, we translate our register automata into EFSMs by introducing outputs: \texttt{OOK} for transitions that can reach accepting states and \texttt{ONO} for transitions reaching a rejecting sink. Since EFSMs separate input and output states, this translation increases the number of states, as reflected in \Cref{tab:experiments}, where the input models of \ralib~are larger.
Thus, the comparison with \ralib~is not meant as a direct performance evaluation, but as a demonstration that our approach effectively learns models in a complementary setting. While \ralib\ learns EFSMs using tree queries ($\TQ$s) and approximate equivalence checking, our method uses word-based membership queries and exact equivalence checking.

Table~\ref{tab:experiments} summarizes the results. \ralib\ requires fewer $\TQ$s because tree queries compactly encode multiple words, whereas our method operates on individual words for $\MQ$s. As a result, our total runtime is dominated by $\MQ$s, while \ralib\ spends more time on $\EQ$s. Despite the larger number of $\MQ$s, our approach scales predictably with the size of the target model, as also reflected by the linear growth in memorability ($\Mem$) queries.
Furthermore, \ralib\ employs approximate equivalence checking with probabilistic guarantees, resulting in a small and constant number of $\EQ$s.
Consequently, the learned models produced by \ralib\ do \emph{not} always accept exactly the same language as the target DRAs, as observed in our experiments.
In contrast, our exact equivalence checking ensures correctness but requires more $\EQ$s as the model size increases.
Moreover, the well-typed register model used in our work appears to enable fast equivalence checking in practice.

We also considered register benchmarks from the \ralib~GitHub repository, which contains a collection of models related to those in the Automata Wiki~\cite{Neider2019}\footnote{\url{https://automata.cs.ru.nl/}}. However, to the best of our knowledge, there is no established translation from EFSMs to our register automata model.
To adapt these benchmarks, we combined each EFSM input transition and its corresponding output transition into a single transition.
Since our model does not support multiple parameters as in EFSMs, we instead introduced additional states to process parameters sequentially. 
All states were made accepting.
Nevertheless, certain EFSM constructs remain challenging to encode. For example, a state may have transitions such as \texttt{I\_get(p)} and \texttt{I\_put(p)}, where \texttt{I\_get} and \texttt{I\_put} are input symbols and $p$ is a parameter. In our translation, input action names are abstracted away, which prevents distinguishing such cases based solely on parameters. As a result, this may introduce nondeterminism in the translated models.
After excluding nondeterministic instances, we observed that, with few exceptions such as the alternative bit protocol and password models, the models learned by our tool often accept universal languages and thus collapse to a single state. As these cases provide limited insight for evaluation, we omit them from our experimental results.
\begin{table}[t]
\centering
\caption{Experimental results for \ralib{} and our approach on $L_n$. DRA sizes are reported as numbers of states, and runtimes are given in seconds. All runtimes are total operation times.}
\label{tab:experiments}
\scalebox{0.73}{
\begin{tabular}{lrrrrrrrrrrrrrr}
\toprule
\multirow{2}{*}{Model}
& \multicolumn{2}{c}{Input DRA size}
& \multicolumn{2}{c}{Learned DRA size}
& \multicolumn{2}{c}{\#$\TQ/\MQ$}
& \multicolumn{2}{c}{\#$\EQ$}
& \multicolumn{2}{c}{$\TQ/\MQ$ Time (s)}
& \multicolumn{2}{c}{$\EQ$ Time (s)}
& \multicolumn{2}{c}{Our $\Mem$} \\
\cmidrule(lr){2-3}
\cmidrule(lr){4-5}
\cmidrule(lr){6-7}
\cmidrule(lr){8-9}
\cmidrule(lr){10-11}
\cmidrule(lr){12-13}
\cmidrule(lr){14-15}
& \multicolumn{1}{c}{\ralib} & \multicolumn{1}{c}{Ours}
& \multicolumn{1}{c}{\ralib} & \multicolumn{1}{c}{Ours}
& \multicolumn{1}{c}{\ralib} & \multicolumn{1}{c}{Ours}
& \multicolumn{1}{c}{\ralib} & \multicolumn{1}{c}{Ours}
& \multicolumn{1}{c}{\ralib} & \multicolumn{1}{c}{Ours}
& \multicolumn{1}{c}{\ralib} & \multicolumn{1}{c}{Ours}
& \multicolumn{1}{c}{\#$\Mem$} & \multicolumn{1}{c}{Time} \\
\midrule
$L_2$ & 9 & 4 & 8 & 4 & 34 & 39 & 2 & 2 & $<$1 & $<$1 & 2.19 & $<$1 & 6 &   $<$1\\
$L_3$ & 16 & 6 & 10 & 6 & 49 & 156 & 2 & 4 & $<$1 & $<$1 & 2.57 & $<$1 & 12 &   $<$1\\
$L_4$ & 22 & 8 & 12 & 8 & 68 & 662 & 2 & 8 & $<$1 & $<$1 & 2.91 & $<$1 & 18 &  $<$1 \\
$L_5$ & 28 & 10 & 14 & 10 & 91 & 2,173 & 2 & 14 & $<$1 & $<$1 & 3.07 & $<$1 & 24 &  $<$1 \\
$L_6$ & 34 & 12 & 16 & 12 & 118 & 5,885 & 2 & 22 & $<$1 & $<$1 & 3.37 & $<$1 & 30 &  $<$1 \\
$L_7$ & 40 & 14 & 18 & 14 & 149 & 13,653 & 2 & 32 & $<$1 & $<$1 & 3.59 & $<$1 & 36 &  $<$1 \\
$L_8$ & 46 & 16 & 20 & 16 & 184 & 28,192 & 2 & 44 & $<$1 & 1.25 & 3.76 & $<$1 & 42 &  $<$1 \\
$L_9$ & 52 & 18 & 22 & 18 & 223 & 53,312 & 2 & 58 & $<$1 & 2.58 & 4.04 & $<$1 & 48 &   $<$1\\
$L_{10}$ & 58 & 20 & 24 & 20 & 266 & 94,412 & 2 & 74 & $<$1 & 5.11 & 3.70 & $<$1 & 54 &  $<$1 \\
\bottomrule
\end{tabular}
}
\end{table}

\section{Additional Experiments}\label{app:experiments}


\paragraph*{\textbf{The Family of Languages $L_{\textit{n}}$.}}

In addition to the figures in the main paper, we report in the following table the number of equivalence queries ($\EQ$), membership queries ($\MQ$) and memorability queries.
We also report the total time, which includes the time for the teacher to solve the three types of queries and the total membership query resolving time to learn each automaton.
We can see that the time consumed by resolving membership queries in the teacher dominates runtime in the whole procedure for learning the target automata.

\renewcommand{\arraystretch}{1.2} 
{\small
\setlength{\tabcolsep}{3pt}
\begin{longtable}{
    S[table-format=1.0] 
    S[table-format=2.0]
    S[table-format=5.2]
    S[table-format=7.2]
    S[table-format=9.2]
    S[table-format=7.2]
    r 
    r 
}
\toprule
\textbf{$n$} &
\multicolumn{1}{r}{\textbf{States}} & 
\multicolumn{1}{r}{\textbf{Transitions}} &
\multicolumn{1}{c}{\textbf{EQ}} & 
\multicolumn{1}{c}{\textbf{MQ}} & 
\multicolumn{1}{r}{\textbf{Mem}} & 
\multicolumn{1}{r}{\textbf{MQ time (s)}} &
\multicolumn{1}{r}{\textbf{Total Time (s)}} \\
\midrule
\endfirsthead

\toprule
\textbf{$n$} &
\multicolumn{1}{r}{\textbf{States}} & 
\multicolumn{1}{r}{\textbf{Transitions}} &
\multicolumn{1}{c}{\textbf{EQ}} & 
\multicolumn{1}{c}{\textbf{MQ}} & 
\multicolumn{1}{r}{\textbf{Mem}} & 
\multicolumn{1}{r}{\textbf{MQ time (s)}} & 
\multicolumn{1}{r}{\textbf{Total Time (s)}} \\
\midrule
\endhead

\midrule \multicolumn{8}{r}{\textit{Continued on next page}} \\
\endfoot

\bottomrule
\endlastfoot

1  & 3  & 3   & 2   & 24        & 3   & {<1} & {<1} \\
2  & 4  & 6   & 2   & 39        & 6   & {<1} & {<1} \\
3  & 6  & 12  & 4   & 156       & 12  & {<1} & {<1} \\
4  & 8  & 18  & 8   & 662       & 18  & {<1} & {<1} \\
5  & 10 & 24  & 14  & 2173      & 24  & {<1} & {<1} \\
6  & 12 & 30  & 22  & 5885      & 30  & {<1} & {<1} \\
7  & 14 & 36  & 32  & 13653     & 36  & {<1} & {<1} \\
8  & 16 & 42  & 44  & 28192     & 42  & 1.25 & 1.67 \\
9  & 18 & 48  & 58  & 53312     & 48  & 2.58 & 3.23 \\
10 & 20 & 54  & 74  & 94412     & 54  & 5.11 & 6.08 \\
11 & 22 & 60  & 92  & 157820    & 60  & 9.10 & 10.45 \\
12 & 24 & 66  & 112 & 250559    & 66  & 15.39 & 17.32 \\
13 & 26 & 72  & 134 & 384197    & 72  & 25.21 & 28.00 \\
14 & 28 & 78  & 158 & 570135    & 78  & 38.99 & 42.70 \\
15 & 30 & 84  & 184 & 822663    & 84  & 59.74 & 64.97 \\
16 & 32 & 90  & 212 & 1155200   & 90  & 88.16 & 95.11 \\
17 & 34 & 96  & 242 & 1590956   & 96  & 129.28 & 138.45 \\
18 & 36 & 102 & 274 & 2148216   & 102 & 181.08 & 193.26 \\
19 & 38 & 108 & 308 & 2852259   & 108 & 256.15 & 271.12 \\
20 & 40 & 114 & 344 & 3730170   & 114 & 347.37 & 368.70 \\
21 & 42 & 120 & 382 & 4814609   & 120 & 473.87 & 499.51 \\
22 & 44 & 126 & 422 & 6133333   & 126 & 615.41 & 647.93 \\
23 & 46 & 132 & 464 & 7730804   & 132 & 823.10 & 863.34 \\
24 & 48 & 138 & 508 & 9644939   & 138 & 1061.34 & 1111.84 \\
25 & 50 & 144 & 554 & 11922628  & 144 & 1369.29 & 1429.55 \\
\end{longtable}
}

\paragraph*{\textbf{The Language $L_{\textit{mid}}$.}}
Recall that $L_{\mathit{mid}} = \{\,a_1 a_2 a_3 \mid a_1 < a_3 < a_2\,\}$ over $(\mathbb{Z},<)$.
In our experiments, we learn $L_{\mathit{mid}}$ over $(\mathbb{R},<)$ instead. 
The minimal well-typed DRA for $L_{\mathit{mid}}$ has 5 states and 11 transitions.
Learning this automaton requires 2 $\EQ$s, 72 $\MQ$s, and 11 memorability queries. 
The final observation table is shown below.

\begin{table}[h]
\centering
\begin{tabular}{l||l|l|c}
& $M$ & $\varepsilon$ & $0 \cdot 1 \cdot 0.5$  \\
\hline
$\varepsilon$ & $\varepsilon$ & $-$ & $+$ \\
$0$ & $0$ & $-$ & $-$ \\
$0\cdot 1$ & $0\cdot 1$ & $-$ & $-$ \\
$0\cdot 1\cdot 0.5$ & $\varepsilon$ & $+$ & $-$ \\
$0\cdot -1$ & $\varepsilon$ & $-$ & $-$ \\
\hline
\end{tabular}
\end{table}

The DRA learned is shown in \Cref{fig:Lmid}. 
Notice that this figure is generated by our tool, and the position index in $E$ starts from $0$ rather than $1$.
\begin{figure}[ht]
    \centering
    \includegraphics[width=\linewidth]{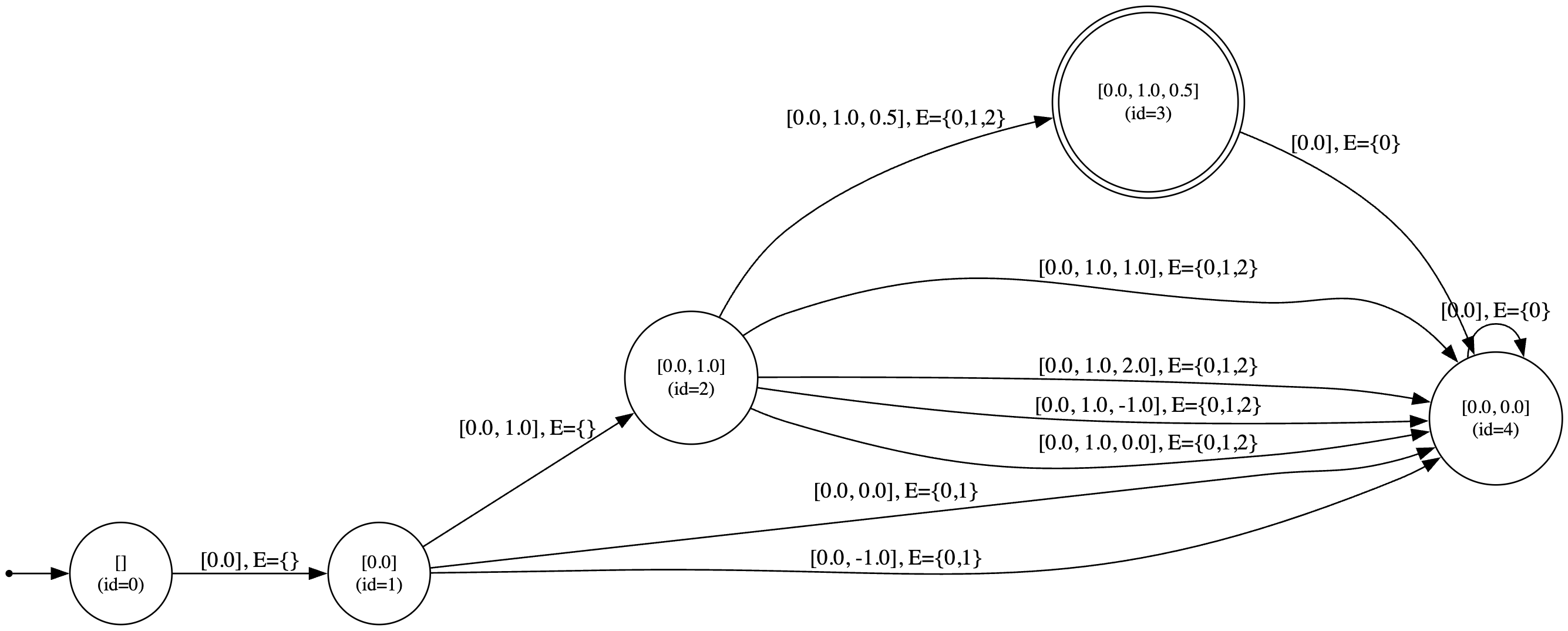}
    \caption{Learned DRA recognizing $L_{\textit{mid}}$}
    \label{fig:Lmid}
\end{figure}
    
\paragraph*{\textbf{The Language $L_{\textit{rep}}$ of \Cref{sec:learning-mem}.}}
Recall that $L_{\mathit{rep}} = \{xyxy \in \alphabetQ^4 \mid x < y \}$.
The minimal well-typed DRA for $L_{\mathit{rep}}$ has 6 states and 14 transitions.
Learning this automaton requires 3 $\EQ$s, 156 $\MQ$s, and 14 memorability queries. 
The final observation table and the DRA are the same as the ones in \Cref{fig:table-and-dra}.

\paragraph*{\textbf{Randomly Generated DRAs.}}
\begin{figure}[htb]
    \centering
    \begin{subfigure}[h]{0.49\textwidth}
        \centering
        \includegraphics[width=\linewidth]{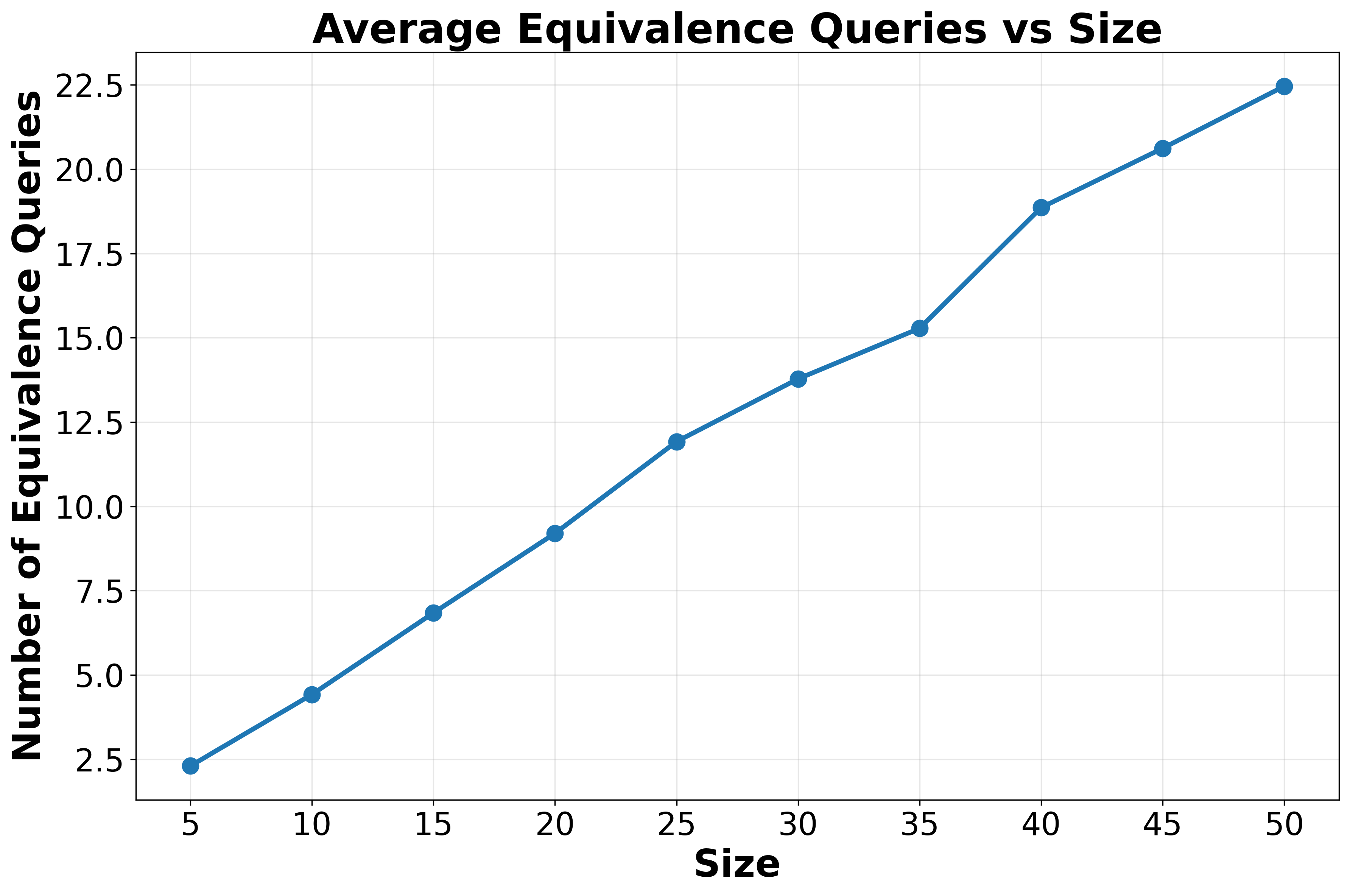}
        \label{fig:fig1}
    \end{subfigure}
    \hfill
    \begin{subfigure}[h]{0.49\textwidth}
        \centering
        \includegraphics[width=\linewidth]{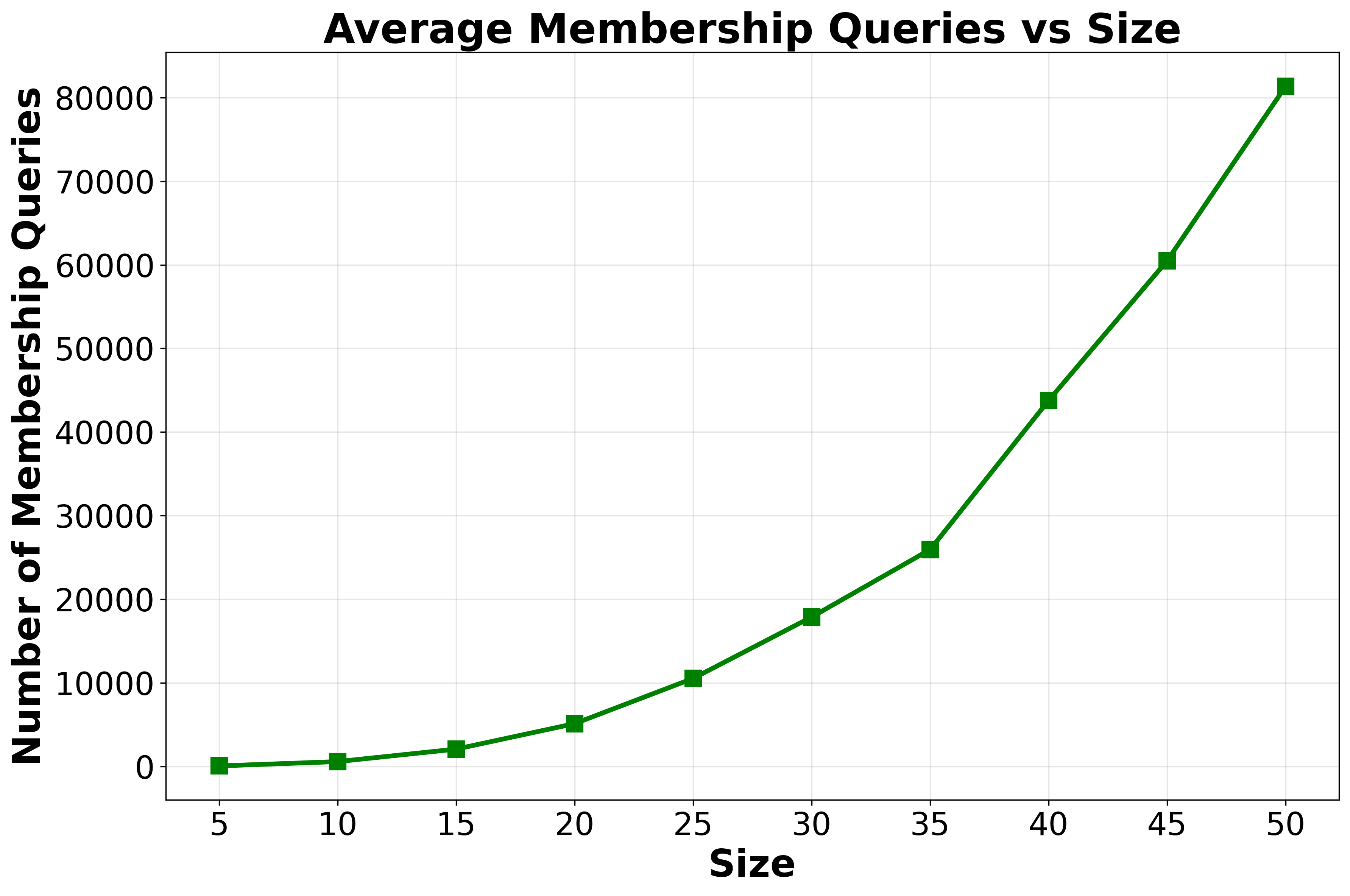}
        \label{fig:fig2}
    \end{subfigure}
    ~
    \begin{subfigure}[h]{0.49\textwidth}
        \centering
        \includegraphics[width=\linewidth]{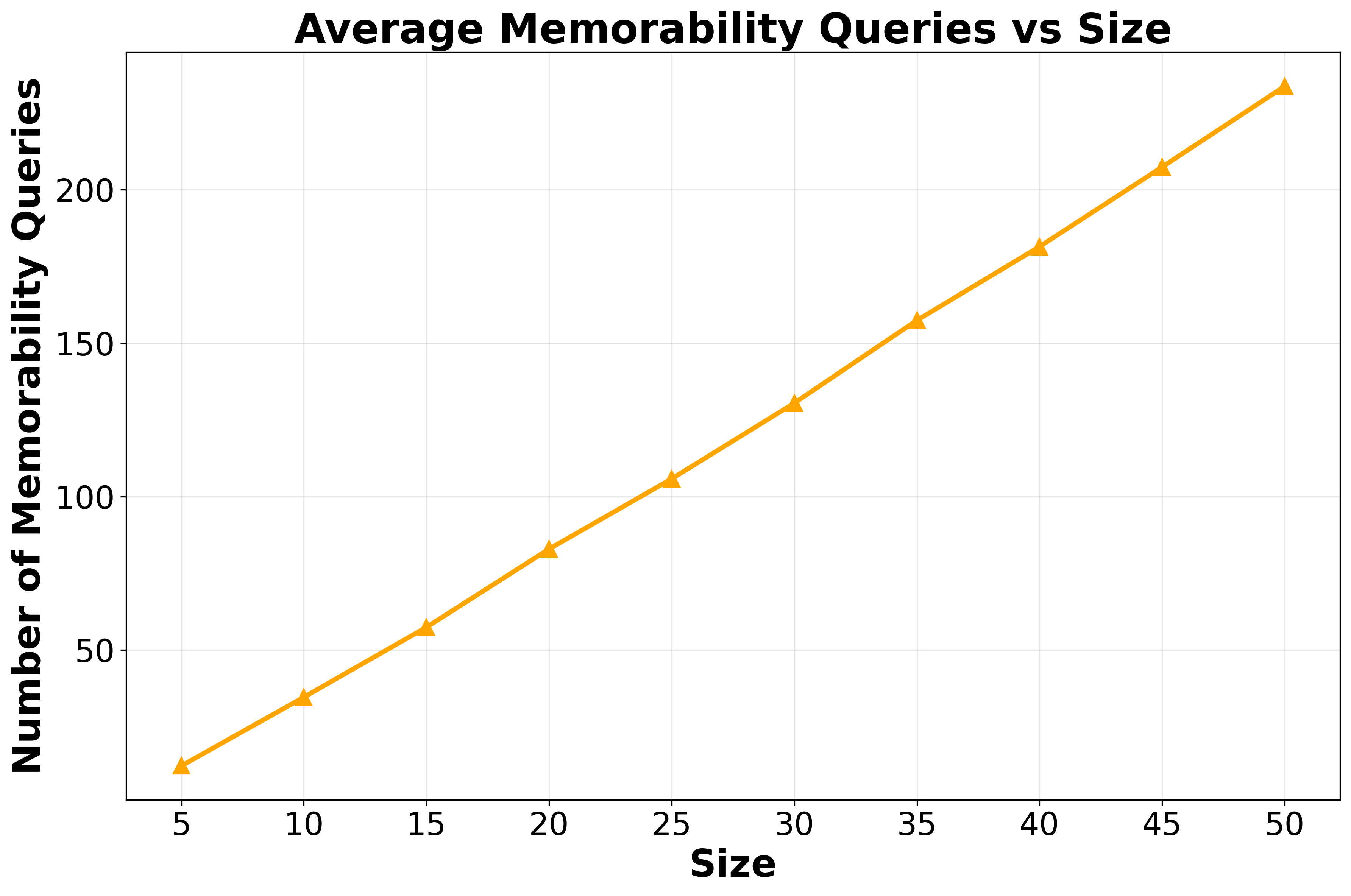}
        \label{fig:fig3}
    \end{subfigure}
    \hfill
    \begin{subfigure}[h]{0.49\textwidth}
        \centering
        \includegraphics[width=\linewidth]{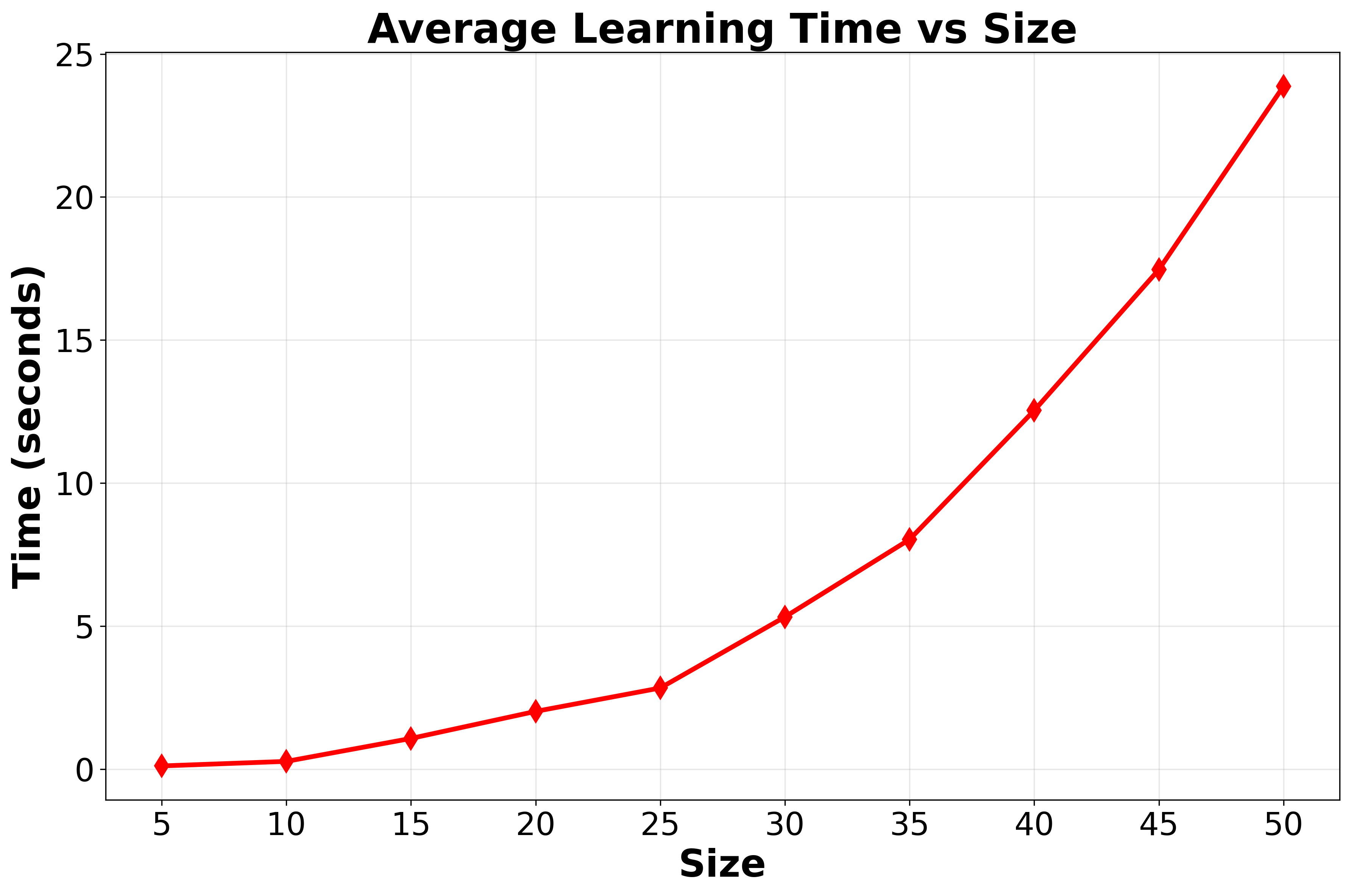}
        \label{fig:fig4}
    \end{subfigure}
    \caption{Average statistics for learning random minimal well-typed DRAs.}
    \label{fig:random_DRA}
\end{figure}

We generated random minimal well-typed DRAs of sizes $5, 10, \ldots, 50$, with 50 automata for each size. 
We report in the following table the \textbf{average} (Avg) number of equivalence queries ($\EQ$), membership queries ($\MQ$), and memorability queries and the average time required.
As can be seen in \Cref{fig:random_DRA}, similar to the trend for the family of languages $L_n$, the number of $\MQ$s grows the fastest, compared to $\EQ$s and then $\Mem$ queries, in accordance with \Cref{thm:correctness-learning-algorithm}.
The average time per automaton here includes the time for the teacher to solve the queries.
{
\newcolumntype{Y}{>{\centering\arraybackslash}X} 
\begin{table}[h]
\centering
\begin{tabularx}{0.97\textwidth}{ 
  l
  S[table-format=9.2]
  S[table-format=6.2]
  S[table-format=9.2]
  S[table-format=8.2]
  S[table-format=8.2]
  }
\toprule
\textbf{Size} & \multicolumn{1}{r}{Avg \#\textbf{Transitions}} &\multicolumn{1}{r}{Avg \#\textbf{EQ}} & \multicolumn{1}{r}{Avg \#\textbf{MQ}} & \multicolumn{1}{r}{Avg \#\textbf{Mem}} 
& \multicolumn{1}{r}{Avg \textbf{Time (s)}} \\
\midrule
5    &  12.2  &     2.3    &    75.0      &      12.2    &   0.12      \\     
10   &   34.6   &    4.4    &    593.1     &      34.6    &   0.28       \\    
15   &   57.4   &    6.8    &    2086.1    &      57.4    &   1.08        \\   
20   &   82.9   &    9.2    &    5131.1    &      82.9    &   2.02        \\   
25   &   105.8  &    11.9   &    10542.2   &      105.8   &   2.84         \\  
30   &   130.5  &    13.8   &    17883.5   &      130.5   &   5.33         \\  
35   &   157.4  &    15.3   &    25949.8   &      157.4   &   8.03         \\  
40   &    181.4  &    18.9   &    43782.9   &      181.4   &   12.54        \\  
45   &     207.4  &    20.6   &    60489.1   &      207.4   &   17.46        \\  
50   &     233.7  &    22.5   &    81354.0   &      233.7   &   23.87  \\
\bottomrule
\end{tabularx}
\end{table}
}

\newpage
\section{Missing Proofs of \Cref{sec:decision-problems}} \label{app:decision-problems}
\paragraph*{\textbf{Language Intersection Non-emptiness and Language Inclusion}.}
Given a $k_1$-DRA $\mathcal{A}_1 = (Q^1, q_0^1,F^1,\Delta^1)$ and a $k_2$-DRA $\mathcal{A}_2 = (Q^2, q_0^2,F^2,\Delta^2)$, the \emph{language intersection non-emptiness} asks whether there is a word $w \in L(\mathcal{A}_1) \cap L(\mathcal{A}_2)$. 
The  \emph{language inclusion} problem asks whether $L(\mathcal{A}_1) \subseteq L(\mathcal{A}_2)$. 
This is equivalent to the language intersection emptiness problem, that is, whether $L(\mathcal{A}_1) \cap \overline{L(\mathcal{A}_2)} = \emptyset$, where $\overline{L(\mathcal{A}_2)}$ denotes the complement of $L(\mathcal{A}_2)$. 
Since $\mathcal{A}_2$ is a DRA, it is straightforward to construct a DRA $\mathcal{B}_2$ accepting the complement language, that is, $L(\mathcal{B}_2) = \overline{L(\mathcal{A}_2)}$, by simply swapping the accepting and non-accepting states of $\mathcal{A}_2$.


To solve the language intersection emptiness problem, we can use the product register automaton of \Cref{def:product-automaton}, which simulates the operations of both automata simultaneously.
Recall that this product register automaton differs from the RAs in \Cref{def:RA}: its registers may contain duplicate values and it is not well-typed.
Nevertheless, language non-emptiness can be checked by simply testing whether there exists a path from the initial state $(q_0^1, q_0^2)$ to some accepting state in $F^\times$.
Although such a path may be exponentially long—bounded by $|Q_1|\cdot |Q_2|\cdot (k_1 + k_2)!$—we do not need to store the entire path. 
It suffices to maintain a counter together with the current configuration of the product automaton, both of which require only polynomial space. 
This yields a PSPACE decision procedure.

For the lower-bound, we reduce from the PSPACE-complete problem of deciding whether a linear-bounded deterministic Turing machine $M$ accepts an input $w$.
The reduction closely follows \cite[Theorem~32]{DBLP:conf/mfcs/MurawskiRT18}, which considers DRAs that do not permit register permutations.
The key is to encode the tape contents using DRA registers, which cannot store duplicate values, and to permute the register symbols while simulating the Turing machine.




\thmInclusionPSPACE*
\begin{proof}
Since DRAs can be easily complemented, it suffices to show that the intersection \emph{non-emptiness} problem for DRAs is PSPACE-complete. 
The complement problem---the intersection \emph{emptiness} problem---is also PSPACE-complete. 
Consequently, the DRA language inclusion problem, which is equivalent to intersection emptiness, is PSPACE-complete as well.
\paragraph*{\textbf{Inclusion in $\mathbf{PSPACE}$.}}
Given a $k_1$-DRA $\mathcal{A}_1 = (Q^1, q_0^1,F^1,\Delta^1)$ and a $k_2$-DRA $\mathcal{A}_2 = (Q^2, q_0^2,F^2,\Delta^2)$, 
we ask whether whether there is a word $w \in L(\mathcal{A}_1) \cap L(\mathcal{A}_2)$. 
Equivalently, it asks whether there is a run from the initial state $(q_0^1, q_0^2)$ to a final state in $F^\times$ in the product automaton $\mathcal{A}_1 \times \mathcal{A}_2$. 

We now consider the product automaton $\mathcal{A}_1 \times \mathcal{A}_2$.
We show that if there is a run from the initial state $(q_0^1, q_0^2)$ to a final state in $F^\times$, then there exists a run of length at most 
$|Q_1|\cdot |Q_2|\cdot (k_1 + k_2)!$ reaching a state in $F^\times$.

Based on this observation, we obtain a nondeterministic polynomial-space algorithm that guesses the input word symbol by symbol. 
The algorithm keeps only a counter and the current configuration of the product automaton, and therefore uses polynomial space.
By Savitch's theorem, PSPACE = NPSPACE, hence the intersection \emph{non-emptiness} problem for DRAs is in PSPACE.

At a configuration $(q^1, q^2, x^1, x^2)$, the algorithm guesses a symbol $a$. 
The number of possible choices for $a$ is bounded by
\[
2(|x^1| + |x^2|) + 1 \;\le\; 2(k_1 + k_2) + 1.
\]
To justify this, we reorder the concatenation $x^1 x^2$ into non-decreasing order and retain only the distinct symbols, denoting the resulting string by 
$x = x_1 \cdots x_\ell$.
The guessed symbol $a$ may then be chosen as:
\begin{itemize}
    \item $x_i$ for any $1 \le i \le \ell$,
    \item a symbol greater than $x_\ell$ or smaller than $x_1$, or
    \item a symbol strictly between $x_i$ and $x_{i+1}$ for any $1 \le i \le \ell - 1$.
\end{itemize}

After guessing $a$, the algorithm transitions to the next configuration and increments the counter. 
It starts with the initial configuration $(q_0^1, q_0^2, \varepsilon, \varepsilon)$ with the counter set to~$0$, and proceeds as follows:
\begin{itemize}
    \item If the counter has not yet reached 
    $|Q_1|\cdot |Q_2|\cdot (k_1 + k_2)!$ 
    and a final state is reached, it accepts.
    \item If the counter has not reached this bound and no final state has been reached, it continues by guessing another symbol.
    \item If the counter reaches 
    $|Q_1|\cdot |Q_2|\cdot (k_1 + k_2)!$ 
    without reaching a final state, it rejects.
\end{itemize}

It remains to show that $|Q_1|\cdot|Q_2|\cdot(k_1+k_2)!$ is an upper bound on the length of a shortest run reaching a final state.
Assume that there is a run in the product automaton from the initial state
$(q_0^1, q_0^2)$ to a final state in $F^\times$ over a word $u$ of length $\ell$.
Let the run be
\[
\rho_0 = (q_0^1, q_0^2, \varepsilon, \varepsilon)
  \xrightarrow{u_1} \rho_1 = (q_1^1, q_1^2, x_1^1, x_1^2)
  \xrightarrow{u_2} \cdots
  \xrightarrow{u_\ell} \rho_{\ell} = (q_{\ell}^1, q_{\ell}^2, x_{\ell}^1, x_{\ell}^2),
\]
where $(q_{\ell}^1, q_{\ell}^2) \in F^\times$.

We say that two configurations 
\[
\rho_i = (q_i^1, q_i^2, x_i^1, x_i^2)
\qquad\text{and}\qquad
\rho_j = (q_j^1, q_j^2, x_j^1, x_j^2)
\]
satisfy $\rho_i \approx \rho_j$ if and only if
$(q_i^1, q_i^2) = (q_j^1, q_j^2)$ and
$x_i^1 x_i^2 \sim_R x_j^1 x_j^2$.

Let $\ell \ge |Q_1|\cdot|Q_2|\cdot (k_1 + k_2)!$.
By the pigeonhole principle, there exist two distinct configurations along the run
such that $\rho_i \approx \rho_j$.
Fix such a pair with $i < j$ so that no earlier–later pair 
$\rho_k, \rho_{k'}$ with $\rho_k \approx \rho_{k'}$ satisfies
either $k \le i < j < k'$ or $k < i \le j \le k'$.

We now construct a new run on a new word that still reaches a final
state in $F^\times$, but in which no two configurations are $\approx$-equivalent.
This new run will therefore have length at most 
$|Q_1|\cdot|Q_2|\cdot (k_1 + k_2)!$.

The idea is to modify the run (and word) from configuration $\rho_i$ onward so that
the $k$-th configuration after $\rho_i$ is $\rho_{j+k}'$ with 
$\rho_{j+k}' \approx \rho_{j+k}$ for all $0 \le k \le (\ell - j)$.
Thus the new run is of the form
\[
\rho_0 \;\cdots\; \rho_i = \rho_j'
   \xrightarrow{u_{j+1}'} \rho_{j+1}'
   \xrightarrow{u_{j+2}'} \cdots
   \xrightarrow{u_\ell'} \rho_{\ell}'.
\]

It remains to show the following key property:
for any transition
$
\rho_i \xrightarrow{a} \rho_{i+1}
$ and any configuration $\rho_j \approx \rho_i$,
there exists a symbol $a'$ such that
$\rho_j \xrightarrow{a'} \rho_{j+1}$ and $\rho_{j+1} \approx \rho_{i+1}$.
Let $\rho_i = (q_i^1, q_i^2, x_i^1, x_i^2)$ and
$\rho_j = (q_j^1, q_j^2, x_j^1, x_j^2)$.
We only need to pick a symbol $a'$ such that
$
x_i^1 x_i^2 a \;\sim_R\; x_j^1 x_j^2 a'
$,
which is always possible because $x_i^1 x_i^2 \;\sim_R\; x_j^1 x_j^2$ and the data domain is dense.

\paragraph*{\textbf{$\mathbf{PSPACE}$-Hardness.}}
For PSPACE-hardness, we reduce from the PSPACE-complete problem of determining whether a linear bounded deterministic Turing
machine $M$ accepts an input $w$ (see, for example, the linear bounded prediction problem in~\cite{DBLP:journals/tcs/GolesMST16}).

More precisely, we consider a Turing machine with binary
alphabet and linearly bounded tape 
$M = \langle Q, q_0, \delta, q_{acc} \rangle$ such that $Q$ is a finite set of states, 
$q_0 \in Q$ is the initial state, $q_{acc} \in Q$ is the accepting state, and 
$\delta : Q \times \{0,1\} \to Q \times \{0,1\} \times \{-1,1\}$ 
is a transition function. 
A configuration of $M$ is a triple $\langle q, i, w \rangle$ where 
$q \in Q$ is the machine state, 
$0 \le i < |M|$ is the head position, and 
$w \in \{0,1\}^{|M|}$ is the tape content. 
Let $\delta(q, w(i)) = \langle q', b, j \rangle$.  
If $1 \le i + j \le |M|$, then the state 
$\langle q', i + j, w[i \mapsto b] \rangle$ 
is the unique successor of $\langle q, i, w \rangle$.
The PSPACE-hard problem asks whether $M$ accepts a word $w_0\in \{0,1\}^{|M|}$, that is, whether $\langle q_0, 1, w_0 \rangle$ reaches a configuration with the accepting state $q_f$.

The main idea of the reduction is similar to the PSPACE-hardness proof of the inclusion problem of a different type of DRAs, the ones where permutations of registers in transitions are not allowed \cite[Theorem~32]{DBLP:conf/mfcs/MurawskiRT18}.
Since the values in the register need to be distinct, we construct two $(|M|+1)$-DRAs $\mathcal{A}_1, \mathcal{A}_2$ such that whenever they synchronize on a data word, their register assignments $\rho_1, \rho_2$ represent the content of $M$ tape cells as follows: $0$ in the $i$-th tape cell is represented by $\rho_1(i) = \rho_2(i)$, and $1$ by $\rho_1(i) \notin \rho_2$. 
One extra register plays a technical role that helps maintain the representation. 
The position of the head and state of the machine are maintained in the state of the automata.
We show that $M$ accepts $w$ if and only if $\mathcal{A}_1$ and $\mathcal{A}_2$ both accept some data word $w'$ (intersection non-emptiness). 

There are two phases for the DRAs: the \emph{initialization} phase and the \emph{simulation} phase. 
In the initialization phase, the DRAs are initialized so that they jointly encode the initial configuration of the Turing machine~$M$, with their combined register assignments representing the initial tape contents~$w_0$. 

\paragraph*{{Initialization Phase.}}
In the initialization phase, both $\mathcal{A}_1$ and $\mathcal{A}_2$ read $2|M|$ symbols, where the $(2i{-}1)$-th and $2i$-th symbols jointly represent $w_0(i)$ for all $1 \le i \le |M|$. 
If $w_0(i)=0$, both automata verify that these two symbols are identical but distinct from all previously seen symbols, and then store the symbol in their respective registers; otherwise, they transition to rejecting sink states. 
If $w_0(i)=1$, the automata ensure that the two symbols are distinct from each other and from all previously seen symbols: $\mathcal{A}_1$ stores the $(2i{-}1)$-th symbol and $\mathcal{A}_2$ stores the $2i$-th symbol. 
Any violation leads to the rejecting sink. 

Thus, whenever neither DRA reaches a rejecting sink, both automata end up in states with register assignments $\rho_1$ and $\rho_2$ that together encode the tape contents $w_0$. 

Each DRA requires $2|M|$ states for this phase, in addition to the initial and rejecting sink states. 
The extra register may be used to check correct initialization: 
$\mathcal{A}_2$ first stores the $(2i{-}1)$-th symbol, and upon reading the $2i$-th symbol, 
if it differs from all previously stored symbols (including the $(2i{-}1)$-th), it updates its register accordingly—moving to a new non-sink state as $w_0(i)=1$ or rejecting otherwise. 
If the two symbols are equal and new, it interprets this as $w_0(i)=0$; otherwise, it transitions to a rejecting sink. 
$\mathcal{A}_1$ can use its extra register in a similar manner.

For example, when $w_0 = 011$, both $\mathcal{A}_1$ and $\mathcal{A}_2$ can read the data word $001234$ and end up in (non-sink) states with register assignments $\rho_1=013$ and  $\rho_2=024$ respectively; the data words $001034$ and $001134$ are rejected by both DRAs and the data word $001232$ is rejected by $\mathcal{A}_2$.

\paragraph*{{Simulation Phase.}}
Once the DRAs are properly initialized, both are in the states representing the initial machine state~$q_0$ and head position~$1$ (i.e., the head on the first tape cell), with register assignments~$\rho_1$ and~$\rho_2$ that together encode the initial tape contents~$w_0$.
At this point, they are ready to simulate the transitions of~$M$.

Each transition $\delta(q, w(i)) = \langle q', b, j \rangle$, that is, $\langle q, i, w \rangle \to \langle q', i+j, w[i \mapsto b] \rangle$, is simulated by at most $2|M|+2$ transitions.
Starting at the states representing the machine state~$q$ and head position~$i$ with register assignments~$\rho_1$ and~$\rho_2$ that jointly encode~$w$, after a sequence of transitions the automata reach the states representing the machine state~$q'$ and head position~$i+j$, with register assignments~$\rho_1'$ and~$\rho_2'$ that jointly encode the updated tape contents~$w[i \mapsto b]$.

The idea is as follows.
The automata read two symbols—the $i$-th symbol in~$\mathcal{A}_1$ and the $i$-th symbol in~$\mathcal{A}_2$—to determine whether $w(i) = 0$ or $1$.
Based on this information, each automaton simulates the corresponding transition of~$M$.
If $M$ does not rewrite $w(i)$, i.e.\ $b = w(i)$, both automata move to new states representing~$q'$ and head position~$i+j$ without changing their register assignments, since $w[i \mapsto b] = w$.

Otherwise, when $b \neq w(i)$, the symbol at position~$i$ must be updated.
Since the register automata can only forget symbols or append new ones at the end of their registers, they first rearrange the registers so that the $i$-th position can be modified.
They do this by reading $2(i-1)$ symbols—$\rho_1(1)\cdots\rho_1(i-1)$ and $\rho_2(1)\cdots\rho_2(i-1)$—so that the registers are in a form where the $i$-th entry is ready to be replaced, reaching temporary assignments 
\[
\begin{aligned}
&\rho_1(i)\cdots\rho_1(|M|)\rho_1(1)\cdots\rho_1(i-1), \text{ and} \\
&\rho_2(i)\cdots\rho_2(|M|)\rho_2(1)\cdots\rho_2(i-1).
\end{aligned}
\]

In the case $w(i) = 0$ and it needs to be updated to $b = 1$, that is, $\rho_1(i) = \rho_2(i)$ and the registers must be updated to two different symbols~$a$ and~$b$, the automata read and store two fresh symbols, reaching states with updated register assignments
\[
\begin{aligned}
&\rho_1(i+1)\cdots\rho_1(|M|)\rho_1(1)\cdots\rho_1(i-1)a, \text{ and}\\
&\rho_2(i+1)\cdots\rho_2(|M|)\rho_2(1)\cdots\rho_2(i-1)b.
\end{aligned}
\]

Conversely, when $w(i) = 1$ and it needs to be updated to $b = 0$, that is, $\rho_1(i) \neq \rho_2(i)$ and both registers must be updated to the same symbol, the automata read a new symbol~$a$, reaching states with updated register assignments
\[
\begin{aligned}
&\rho_1(i+1)\cdots\rho_1(|M|)\rho_1(1)\cdots\rho_1(i-1)a, \text{ and}\\
&\rho_2(i+1)\cdots\rho_2(|M|)\rho_2(1)\cdots\rho_2(i-1)a.
\end{aligned}
\]
Finally, another $2(|M|-i)$ transitions restore the correct register order, yielding 
\[
\rho_1' = \rho_1(i \mapsto a)
\quad\text{and either}\quad
\rho_2' = \rho_2(i \mapsto a)
\ \text{or}\ 
\rho_2' = \rho_2(i \mapsto b).
\]

Both automata are in accepting states when they correspond to the machine accepting state $q_{acc}$.
It is easy to see that an accepting computation of $M$ on $w_0$ corresponds to a data word accepted by both $\mathcal{A}_1$ and $\mathcal{A}_2$, and vice versa.
Since both DRAs are of polynomial size in $M$, the reduction is in polynomial time.
Moreover, the DRAs constructed can use identity tests only. 
\qed
\end{proof}

\paragraph*{\textbf{Configuration Equivalence}.}
Given a DRA $\mathcal{A}$ over $(\Sigma, R)$ and a configuration $(q, v)$ of $\mathcal{A}$, we define the language of $\mathcal{A}^{(q, v)}$ as $L(\mathcal{A}^{(q, v)}) = \{ w \in \Sigma^* \mid \exists\, q_f \in F, u \in \Sigma^*.~(q, v) \xrightarrow{w}_\mathcal{A} (q_f, u) \}$.
Given two DRAs $\mathcal{A}_1$ and $\mathcal{A}_2$ over $(\Sigma, R)$ and two configurations $(q_1, v_1)$ and $(q_2, v_2)$ of $\mathcal{A}_1$ and $\mathcal{A}_2$ respectively, the configuration equivalence problem asks whether $L(\mathcal{A}_1^{(q_1, v_1)}) = L(\mathcal{A}_2^{(q_2, v_2)})$.



\paragraph*{\textbf{Language Equivalence}.}
Given two DRAs $\mathcal{A}_1$ and $\mathcal{A}_2$ over $(\Sigma, R)$, the language equivalence problem asks whether $L(\mathcal{A}_1) = L(\mathcal{A}_2)$. 

The language equivalence problem is a special case of the configuration equivalence problem, that is, it is equivalent to the configuration equivalence problem whether $L(\mathcal{A}_1^{(q_0^1, \varepsilon)}) = L(\mathcal{A}_2^{(q_0^2, \varepsilon)})$ where $(q_0^1, \varepsilon)$ and $(q_0^2, \varepsilon)$ are the initial configurations of $\mathcal{A}_1$ and $\mathcal{A}_2$ respectively.
Both problems are in PSPACE, by arguments analogous to the PSPACE upper-bound proof for language inclusion.
Indeed, language equivalence reduces to checking the two inclusions $L(\mathcal{A}_1) \subseteq L(\mathcal{A}_2)$ and $L(\mathcal{A}_2) \subseteq L(\mathcal{A}_1)$, which is then in PSPACE.


\thmLangEquiv*
\begin{proof}
    It suffices to prove that the configuration equivalence problem is in PSPACE since the language equivalence problem is a special case of the configuration equivalence problem.
    Given two DRAs $\mathcal{A}_1$ and $\mathcal{A}_2$ and configurations $(q_1, v_1)$ and $(q_2, v_2)$ of $\mathcal{A}_1$ and $\mathcal{A}_2$, respectively, the configuration equivalence problem
    $L(\mathcal{A}_1^{(q_1, v_1)}) = L(\mathcal{A}_2^{(q_2, v_2)})$
    reduces to checking two inclusions:
    $L(\mathcal{A}_1^{(q_1, v_1)}) \subseteq L(\mathcal{A}_2^{(q_2, v_2)})$ and
    $L(\mathcal{A}_1^{(q_1, v_1)}) \supseteq L(\mathcal{A}_2^{(q_2, v_2)})$.
    Each inclusion further reduces to a language intersection emptiness check, assuming the DRAs start from the given initial configurations.  
    For example, checking $L(\mathcal{A}_1^{(q_1, v_1)}) \subseteq L(\mathcal{A}_2^{(q_2, v_2)})$ reduces to checking
    $L(\mathcal{A}_1^{(q_1, v_1)}) \cap L(\mathcal{B}_2^{(q_2, v_2)}) = \emptyset$,
    where $\mathcal{B}_2$ is the complement of $\mathcal{A}_2$, obtained by swapping its accepting and non-accepting states.  
    This intersection emptiness check is in PSPACE, similar to the case of DRAs without initial assignments.
    \qed
\end{proof}

\paragraph*{\textbf{Memorability}.}
Given a $k$-DRA $\mathcal{A}$ over $(\Sigma, <)$, a word $u \in \Sigma^{*}$ and $a \in u$, the memorability problem asks whether $a$ is $L$-memorable in $u$, that is, whether $a \in mem(u)$.

We first present a simpler version of \Cref{def:memorable} for the case in which  the alphabet $\Sigma$ is dense. 
In \Cref{def:memorable}, the alphabet $\Sigma$ may be dense or non-dense.
When $\Sigma$ is non-dense, we may not find a symbol $b$ and an $R$-preserving mapping $\sigma$ for $u$ with $\sigma(a)=b$ and $u \sim_R \sigma(u)$, hence the definition refers to another word of the same type.
In contrast, for dense $\Sigma$, such a symbol $b$ and mapping $\sigma$ always exist.
We simplify the definition of memorable symbol for dense alphabet as follows. 
\begin{restatable}{lemma}{propDenseMem}\label{prop:dense-mem}
Let $L$ be a language over $(\Sigma, R)$, where $\Sigma$ is dense. A symbol $a \in \Sigma$ is $L$-memorable in a word $u \in \Sigma^*$ if and only if there exist a symbol $b \in \Sigma$, a word $v \in \Sigma^*$ and an $R$-preserving mapping $\sigma$ with $\sigma(a) = b$ and $\sigma(d) = d$ for all other $d \in u \cdot v$, such that $v$ is an \emph{$L$-distinguishing extension} for $u$ and $\sigma(u)$, that is, 
either $u \cdot v \in L$ or $\sigma(u) \cdot v \in L$.
\end{restatable}
\begin{proof}
        First, we consider the only-if direction.
        Let $u' = u$, $a' = a$, $b'= b$.
        We have (1) $u \cdot a = u' \cdot a' \sim_R \sigma(u' \cdot a')$; and (2) either $u \cdot v = u' \cdot v \in L$ or $u \cdot \sigma(v) = u' \cdot \sigma(v) \in L$.

        Now, we consider the if direction.
        Suppose we have a word $u'$ of the same $\sim_R$-type as $u$, symbols $a', b' \in \Sigma$, a word $w \in \Sigma^{*}$ and an $R$-preserving mapping $\sigma$ with $\sigma(a') = b'$ and $\sigma(d) = d$ for all other $d \in u' \cdot w$ such that the two conditions in \Cref{def:memorable} hold. 
        Since $\Sigma$ is dense, we can use $\sigma_{u \to u'}$, the bijective $R$-preserving mapping for $u$ and $u'$.
        Let $b = \sigma_{u \to u'}^{-1}(b')$ and $v = \sigma_{u \to u'}^{-1}(w)$. 
        Let $\sigma'$ be a mapping such that $\sigma'(a) = b$ and $\sigma'(d) = d$ for all other symbols $d \in u \cdot v$.
        We have 
        \begin{align}
          u \cdot v \sim_R u' \cdot w \sim_R \sigma(u' \cdot w) \sim_R \sigma_{u \to u'}^{-1}\big(\sigma(u' \cdot w)\big) = \sigma'(u \cdot w) \label{rel1}\\
         u' \cdot w  \sim_R \sigma_{u \to u'}^{-1}\big(u' \cdot w\big) = u \cdot v \label{rel2}\\
         \sigma(u')\cdot w \sim_R   \sigma_{u \to u'}^{-1}\big(\sigma(u')\cdot w \big)  = \sigma'(u)\cdot v \label{rel3}
        \end{align}
         The equality $\sigma_{u \to u'}^{-1}\big(\sigma(u' \cdot w)\big) = \sigma'(u \cdot w)$ in $\ref{rel1}$ holds since $\sigma$ only changes the symbol $a'$ in $u' \cdot w$ to $b'$ which is then mapped to $b$ by $\sigma_{u \to u'}^{-1}$, while all other symbols will be mapped to the corresponding symbols in $u \cdot v$ by $\sigma_{u \to u'}^{-1}$.
         It follows from $\sigma_{u \to u'}^{-1}\big( w \big)$ that the equality $\sigma_{u \to u'}^{-1}\big(\sigma(u')\cdot w \big)  = \sigma'(u)\cdot v$ in \ref{rel3} holds.
 
        Since we have either $u' \cdot w \in L$ or $\sigma(u') \cdot w \in L$, 
        together with the relations \ref{rel2} and \ref{rel3}, we obtain that either $u \cdot v \in L$ or $\sigma'(u) \cdot v \notin L$.
        \qed
\end{proof}

We use the technical lemma below to establish the PSPACE upper bound.
Informally, the lemma states that it suffices to replace the symbol $a$
with an arbitrary symbol $b$, provided that the resulting word remains
order-equivalent to $u$, and then compare the corresponding residual languages.
\begin{restatable}{lemma}{lemAltMem}\label{lem:alternative-mem}
Let $\mathcal{A}$ be a $k$-DRA over $(\Sigma, <)$, and let $u \in \Sigma^{*}$ and $a \in u$.  
Then $a \in mem_L(u)$ if and only if, for any fresh symbol $b$ and any simple mapping $\sigma$ such that $\sigma(a) = b$ and $\sigma(d) = d$ for all other symbol $d \in \Sigma$, we have 
\[
u \cdot v \in L(\mathcal{A}) \;\iff\; \sigma(u) \cdot v \notin L(\mathcal{A}) \quad \text{for some } v \in \Sigma^{*}.
\]
\end{restatable}
\begin{proof}
Given a word $u \in \Sigma^{*}$ and $a \in u$, assume $a \in mem_{L}(u)$.

By \Cref{prop:dense-mem}, we know there is a fresh symbol  $b \notin u$, a word $v \in \Sigma^*$ and an $R$-preserving mapping $\sigma$ with $\sigma(a) = b$ and $\sigma(d) = d$ for all other symbols $d \in u \cdot v$ such that $v$ is an $L$-distinguishing word for $u$ and $\sigma(u)$.

Let $b' \notin u$ and let $\sigma'$ be a mapping with $\sigma'(a)=b'$ and 
$\sigma'(d)=d$ for all other symbols $d\in \Sigma$ such that 
$u \sim_{R} \sigma'(u)$.  
We want to prove that there is an $L$-distinguishing word $v'$ for 
$u$ and $\sigma'(u)$.  
If $b'=b$, this is immediate by taking $v'=v$.  
For the case $b' \neq b$, we show below that it is always possible to find such a $v'$.

Without loss of generality assume $b'>a$.  
We define a mapping $\sigma''$ that helps to construct $v'$ from $v$.  
The key idea is to map all symbols in $v$ that lie between $a$ and $b'$ to symbols greater than $b'$, while preserving their relative order with respect to all symbols except $b'$.  
Formally, let $w$ be the increasing sequence of symbols in $v$ that satisfy
$a < w[i] \le b'$ for all $i$, and $w[i] < w[j]$ whenever $i < j$.

Let $c \in u \cdot v$ be the smallest symbol strictly larger than $b'$ such that there is no symbol in $u \cdot v$ between $b'$ and $c$.  
We now define $\sigma''$ as follows:
\[
  \sigma''(w[i]) \;=\; 
  b' + \frac{c-b'}{|w|+1} \cdot (i+1)
  \quad \text{for all } w[i],
\]
and $\sigma''(d)=d$ for all other symbols $d\in v$.

We then let $v' = \sigma''(v)$.  
It is not hard to verify that 
\[
  u \cdot v \sim_{R} u \cdot v' 
  \qquad\text{and}\qquad
  \sigma(u \cdot v) \sim_{R} \sigma'(u \cdot v').
\]
This completes the proof.
\qed
\end{proof}



\thmMemPSPACE*
\begin{proof}
    Given a DRA $\mathcal{A} = (Q, q_0, F, \Delta)$, a word $u$, and a symbol $a \in u$, 
    we give a PSPACE algorithm to decide whether $a \in mem_{L}(u)$.
    Let $b \notin u$ be a symbol such that $u \sim_R u'$, where $u'$ is obtained from $u$ by replacing the occurrence of $a$ with $b$.
    By \Cref{lem:alternative-mem}, it suffices to check whether there exists an $L$-distinguishing word $v$ for $u$ and $u'$.
    
    This reduces to checking whether two configurations $(q, r)$ and $(q', r')$ are equivalent, where
    \[
    (q_0, \varepsilon) \xrightarrow{u}_{\mathcal{A}} (q, r)
    \qquad\text{and}\qquad
    (q_0, \varepsilon) \xrightarrow{u'}_{\mathcal{A}} (q', r').
    \]
    By \Cref{thm:equiv-pspace}, configuration equivalence for DRAs can be decided in PSPACE.
    Thus, whether $a$ is $L$-memorable in $u$ can also be checked in PSPACE.
    \qed
\end{proof}
\end{document}